\documentclass[english]{article}
\usepackage[latin9]{inputenc}
\usepackage{verbatim}
\usepackage{units}
\usepackage{amstext}
\usepackage{amsthm}
\usepackage{amsmath}
\usepackage{amssymb}
\usepackage{graphicx}

\makeatletter
\numberwithin{equation}{section}
\theoremstyle{definition}
\newtheorem{defn}{\protect\definitionname}
\theoremstyle{plain}
\newtheorem{thm}{\protect\theoremname}
\theoremstyle{plain}
\newtheorem{cor}{\protect\corollaryname}
\theoremstyle{plain}
\newtheorem*{prop*}{\protect\propositionname}
\theoremstyle{remark}
\newtheorem{rem}{\protect\remarkname}
\theoremstyle{plain}
\newtheorem{prop}{\protect\propositionname}
\theoremstyle{plain}
\newtheorem{lem}{\protect\lemmaname}

\usepackage[hyphens]{url}
\usepackage{iclr2020_conference}
 \iclrfinalcopy

\author{Arthur Jacot, Franck Gabriel \& Cl\'ement Hongler \\
Chair of Statistical Field Theory\\
Ecole Polytechnique F\'ed\'erale de Lausanne\\
\texttt{\{arthur.jacot,franck.grabriel,clement.hongler\}@epfl.ch}
}

\usepackage{hyperref}       

\usepackage{url}            
\usepackage{nicefrac}       
\usepackage{microtype}      

\usepackage{cleveref}

\usepackage{thmtools} 

\allowdisplaybreaks

\makeatother

\usepackage{babel}
\providecommand{\corollaryname}{Corollary}
\providecommand{\definitionname}{Definition}
\providecommand{\lemmaname}{Lemma}
\providecommand{\propositionname}{Proposition}
\providecommand{\remarkname}{Remark}
\providecommand{\theoremname}{Theorem}

\begin{document}
\title{The asymptotic spectrum of the Hessian of DNN throughout training}

\maketitle
\begin{abstract}
The dynamics of DNNs during gradient descent is described by the so-called Neural Tangent Kernel (NTK). In this article, we show that the NTK allows one to gain precise insight into the Hessian of the cost of DNNs. When the NTK is fixed during training, we obtain a full characterization of the asymptotics of the spectrum of the Hessian, at initialization and during training. In the so-called mean-field limit, where the NTK is not fixed during training, we describe the first two moments of the Hessian at initialization.
\end{abstract}

\section{Introduction}

The advent of deep learning has sparked a lot of interest in the loss
surface of deep neural networks (DNN), and in particular its Hessian.
However to our knowledge, there is still no theoretical description
of the spectrum of the Hessian. Nevertheless a number of phenomena
have been observed numerically.

The loss surface of neural networks has been compared to the energy
landscape of different physical models \citep{Choromanska,Geiger18,Mei2018}.
It appears that the loss surface of DNNs may change significantly
depending on the width of the network (the number of neurons in the
hidden layer), motivating the distinction between the under- and over-parametrized
regimes \citep{Baity18,Geiger18,geiger2019}.

The non-convexity of the loss function implies the existence of a
very large number of saddle points, which could slow down training.
In particular, in \citep{Pascanu2014,Dauphin2014}, a relation between
the rank of saddle points (the number of negative eigenvalues of the
Hessian) and their loss has been observed.

For overparametrized DNNs, a possibly more important phenomenon is
the large number of flat directions \citep{Baity18}. The existence
of these flat minima is conjectured to be related to the generalization
of DNNs and may depend on the training procedure \citep{Hochreiter97,Chaudhari16,Wu2017}.

In \citep{jacot2018neural} it has been shown, using a functional approach,
that in the infinite-width limit, DNNs behave like kernel methods
with respect to the so-called Neural Tangent Kernel, which is determined
by the architecture of the network. This leads to convergence guarantees for DNNs \citep{jacot2018neural,Du2019,Allen-Zhu2018,huang2019NTH} and strengthens the connections
between neural networks and kernel methods \citep{Neal1996,Cho2009,Lee2017}.

Our approach also allows one to probe the so-called mean-field/active limit (studied in \citep{Rotskoff2018,Chizat2018,Mei2018} for shallow networks), where the NTK varies during training.

This raises the question: can we use these new results to gain insight
into the behavior of the Hessian of the loss of DNNs, at least in
the small region explored by the parameters during training?

\subsection{Contributions}

Following ideas introduced in \citep{jacot2018neural}, we consider
the training of $L+1$-layered DNNs in a functional setting. For a
functional cost $\mathcal{C}$, the Hessian of the loss $\mathbb{R}^{P}\ni\theta\mapsto\mathcal{C}\left(F^{\left(L\right)}\left(\theta\right)\right)$
is the sum of two $P\times P$ matrices $I$ and $S$. We show the
following results for large $P$ and for a fixed number of datapoints
$N$: 
\begin{itemize}
\item The first matrix $I$ is positive semi-definite and its eigenvalues
are given by the (weighted) kernel PCA of the dataset with respect
to the NTK. The dominating eigenvalues are the principal components
of the data followed by a high number of small eigenvalues. The ``flat
directions'' are spanned by the small eigenvalues and the null-space
(of dimension at least $P-N$ when there is a single output). Because
the NTK is asymptotically constant \citep{jacot2018neural}, these
results apply at initialization, during training and at convergence. 
\item The second matrix $S$ can be viewed as residual contribution to $H$,
since it vanishes as the network converges to a global minimum. We
compute the limit of the first moment $\mathrm{Tr}\left(S\right)$
and characterize its evolution during training, of the second moment
$\mathrm{Tr}\left(S^{2}\right)$ which stays constant during training,
and show that the higher moments vanish.
\item Regarding the sum $H=I+S$, we show that the matrices $I$ and $S$
are asymptotically orthogonal to each other at initialization and
during training. In particular, the moments of the matrices $I$ and
$S$ add up: $tr(H^{k})\approx tr(I^{k})+tr(S^{k})$. 
\end{itemize}
These results give, for any depth and a fairly general non-linearity,
a complete description of the spectrum of the Hessian in terms of
the NTK at initialization and throughout training. Our theoretical
results are consistent with a number of observations about the Hessian \citep{Hochreiter97,Pascanu2014,Dauphin2014,Chaudhari16,Wu2017,Pennington2017,Geiger18},
and sheds a new light on them. 

\subsection{Related works}

The Hessian of the loss has been studied through the decomposition
$I+S$ in a number of previous works \citep{Sagun,Pennington2017,Geiger18}.

For least-squares and cross-entropy costs, the first matrix $I$ is
equal to the Fisher matrix \citep{Wagenaar1998,Pascanu2013}, whose
moments have been described for shallow networks in \citep{Pennington2018}.
For deep networks, the first two moments and the operator norm of
the Fisher matrix for a least squares loss were computed at initialization
in \citep{Karakida2018} conditionally on a certain independence assumption;
our method does not require such assumptions. Note that their approach
implicitly uses the NTK.%

The second matrix $S$ has been studied in \citep{Pennington2017,Geiger18}
for shallow networks, conditionally on a number of assumptions. Note
that in the setting of \citep{Pennington2017}, the matrices $I$ and
$S$ are assumed to be freely independent, which allows them to study
the spectrum of the Hessian; in our setting, we show that the two
matrices $I$ and $S$ are asymptotically orthogonal to each other.

\section{Setup}

We consider deep fully connected artificial neural networks (DNNs)
using the setup and NTK parametrization of \citep{jacot2018neural},
taking an arbitrary nonlinearity $\sigma\in C_{b}^{4}(\mathbb{R})$
(i.e. $\sigma:\mathbb{R}\to\mathbb{R}$ that is 4 times continuously
differentiable function with all four derivatives bounded). The layers
are numbered from $0$ (input) to $L$ (output), each containing $n_{\ell}$
neurons for $\ell=0,\ldots,L$. The $P=\sum_{\ell=0}^{L-1}\left(n_{\ell}+1\right)n_{\ell+1}$
parameters consist of the weight matrices $W^{(\ell)}\in\mathbb{R}^{n_{\ell+1}\times n_{\ell}}$
and bias vectors $b^{(\ell)}\in\mathbb{R}^{n_{\ell+1}}$ for $\ell=0,\ldots,L-1$.
We aggregate the parameters into the vector $\theta\in\mathbb{R}^{P}$. 

The activations and pre-activations of the layers are defined recursively
for an input $x\in\mathbb{R}^{n_{0}}$, setting $\alpha^{(0)}(x;\theta)=x$
: 
\begin{align*}
\tilde{\alpha}^{(\ell+1)}(x;\theta) & =\frac{1}{\sqrt{n_{\ell}}}W^{(\ell)}\alpha^{(\ell)}(x;\theta)+\beta b^{(\ell)},\\
\alpha^{(\ell+1)}(x;\theta) & =\sigma\big(\tilde{\alpha}^{(\ell+1)}(x;\theta)\big).
\end{align*}
The parameter $\beta$ is added to tune the influence of the bias
on training\footnote{In our experiments, we take $\beta=0.1$.}. All
parameters are initialized as iid $\mathcal{N}(0,1)$ Gaussians.

We will in particular study the network function, which maps inputs
$x$ to the activation of the output layer (before the last non-linearity):
\[
f_{\theta}(x)=\tilde{\alpha}^{(L)}(x;\theta).
\]

In this paper, we will study the limit of various objects as \emph{$n_{1},\ldots,n_{L}\to\infty$
sequentially}, i.e. we first take $n_{1}\to\infty,$ then $n_{2}\to\infty$,
etc. This greatly simplifies the proofs, but they could in principle
be extended to the simultaneous limit, i.e. when $n_{1}=...=n_{L-1}\to\infty$.
All our numerical experiments are done with `rectangular' networks
(with $n_{1}=...=n_{L-1}$) and match closely the predictions for
the sequential limit.

In the limit we study in this paper, the NTK is asymptotically fixed,
as in \citep{jacot2018neural,Allen-Zhu2018,Du2019,exact_arora2019,huang2019NTH}.
By rescaling the outputs of DNNs as the width increases, one can reach
another limit where the NTK is not fixed \citep{Chizat2018,chizat2018note,Rotskoff2018,mei2019mean}.
Some of our results can be extended to this setting, but only at initialization (see Section \ref{sec:meanfield}). The behavior during training becomes however much more complex.

\subsection{Functional viewpoint}

The network function lives in a function space $f_{\theta}\in\mathcal{F}:=\left[\mathbb{R}^{n_{0}}\to\mathbb{R}^{n_{L}}\right]$
and we call the function $F^{(L)}:\mathbb{R}^{P}\to\mathcal{F}$ that
maps the parameters $\theta$ to the network function $f_{\theta}$
the \emph{realization function}. We study the differential behavior
of $F^{(L)}$: 
\begin{itemize}
\item The derivative $\mathcal{D}F^{(L)}\in\mathbb{R}^{P}\otimes\mathcal{F}$
is a function-valued vector of dimension $P$. The $p$-th entry $\mathcal{D}_{p}F^{(L)}=\partial_{\theta_{p}}f_{\theta}\in\mathcal{F}$
represents how modifying the parameter $\theta_{p}$ modifies the
function $f_{\theta}$ in the space $\mathcal{F}$. 
\item The Hessian $\mathcal{H}F^{(L)}\in\mathbb{R}^{P}\otimes\mathbb{R}^{P}\otimes\mathcal{F}$
is a function-valued $P\times P$ matrix.
\end{itemize}
The network is trained with respect to the cost functional: 
\[
\mathcal{C}(f)=\frac{1}{N}\sum_{i=1}^{N}c_{i}\left(f(x_{i})\right),
\]
for strictly convex $c_{i}$, summing over a finite dataset $x_{1},\ldots,x_{N}\in\mathbb{R}^{n_{0}}$
of size $N$. The parameters are then trained with gradient descent
on the composition $\mathcal{C}\circ F^{(L)}$, which defines the
usual loss surface of neural networks.

In this setting, we define the finite realization function $Y^{(L)}:\mathbb{R}^{P}\to\mathbb{R}^{Nn_{L}}$
mapping parameters $\theta$ to be the restriction of the network
function $f_{\theta}$ to the training set $y_{ik}=f_{\theta,k}(x_{i})$.
The Jacobian $\mathcal{D}Y^{(L)}$ is hence an $Nn_{L}\times P$ matrix
and its Hessian $\mathcal{H}Y^{(L)}$ is a $P\times P\times Nn_{L}$
tensor. Defining the restricted cost $C(y)=\frac{1}{N}\sum_{i}c_{i}(y_{i})$,
we have $\mathcal{C}\circ F^{(L)}=C\circ Y^{(L)}$.

For our analysis, we require that the gradient norm $\left\Vert \mathcal{D}C\right\Vert $
does not explode during training. The following condition is sufficient:
\begin{defn}
A loss $C:\mathbb{R}^{Nn_{L}}\to\mathbb{R}$ has bounded gradients
over sublevel sets (BGOSS) if the norm of the gradient is bounded
over all sets $U_{a}=\left\{ Y\in\mathbb{R}^{Nn_{L}}:C(Y)\leq a\right\} $.

For example, the Mean Square Error (MSE) $C(Y)=\frac{1}{2N}\left\Vert Y^{*}-Y\right\Vert ^{2}$
for the labels $Y^{*}\in\mathbb{R}^{Nn_{L}}$ has BGOSS because $\left\Vert \nabla C(Y)\right\Vert ^{2}=\frac{1}{N}\left\Vert Y^{*}-Y\right\Vert ^{2}=2C(Y)$.
For the binary and softmax cross-entropy the gradient is uniformly
bounded, see Proposition \ref{prop:bounded-gradient-cross-entropy}
in Appendix \ref{sec:Proofs}.
\end{defn}

\subsection{Neural Tangent Kernel}

The behavior during training of the network function $f_{\theta}$
in the function space $\mathcal{F}$ is described by a (multi-dimensional)
kernel, the \emph{Neural Tangent Kernel} (NTK) 
\[
\Theta_{k,k'}^{(L)}(x,x')=\sum_{p=1}^{P}\partial_{\theta_{p}}f_{\theta,k}(x)\partial_{\theta_{p}}f_{\theta,k'}(x').
\]

During training, the function $f_{\theta}$ follows the so-called
\emph{kernel gradient descent} with respect to the NTK, which is defined
as 
\begin{align*}
\partial_{t}f_{\theta(t)}(x) & =-\nabla_{\Theta^{(L)}}C_{|f_{\theta(t)}}(x):=-\frac{1}{N}\sum_{i=1}^{N}\Theta^{(L)}(x,x_{i})\nabla c{}_{i}(f_{\theta(t)}(x_{i})).
\end{align*}

In the infinite-width limit (letting $n_{1}\to\infty,\ldots,n_{L-1}\to\infty$
sequentially) and for losses with BGOSS, the NTK converges to a deterministic
limit $\Theta^{(L)}\to\Theta_{\infty}^{(L)}\otimes Id_{n_{L}}$, which
is constant during training, uniformly on finite time intervals $\left[0,T\right]$
\citep{jacot2018neural}. For the MSE loss, the uniform
convergence of the NTK was proven for $T=\infty$ in \citep{exact_arora2019}.

The limiting NTK $\Theta_{\infty}^{(L)}:\mathbb{R}^{n_{0}}\times\mathbb{R}^{n_{0}}\to\mathbb{R}$
is constructed as follows:
\begin{enumerate}
\item For $f,g:\mathbb{R}\to\mathbb{R}$ and a kernel $K:\mathbb{R}^{n_{0}}\times\mathbb{R}^{n_{0}}\to\mathbb{R}$,
define the kernel $\mathbb{L}_{K}^{f,g}:\mathbb{R}^{n_{0}}\times\mathbb{R}^{n_{0}}\to\mathbb{R}$
by 
\[
\mathbb{L}_{K}^{f,g}(x_{0},x_{1})=\mathbb{E}_{(a_{0},a_{1})}\left[f(a_{0})g(a_{1})\right],
\]
for $(a_{0},a_{1})$ a centered Gaussian vector with covariance matrix
$\left(K(x_{i},x_{j})\right)_{i,j=0,1}$. For $f=g$, we denote by
$\mathbb{L}_{K}^{f}$ the kernel $\mathbb{L}_{K}^{f,f}$. 
\item We define the kernels $\Sigma_{\infty}^{(\ell)}$ for each layer of
the network, starting with $\Sigma_{\infty}^{(1)}(x_{0},x_{1})=\nicefrac{1}{n_{0}}(x_{0}^{T}x_{1})+\beta^{2}$
and then recursively by $\Sigma_{\infty}^{(\ell+1)}=\mathbb{L}_{\Sigma_{\infty}^{(\ell)}}^{\sigma}+\beta^{2}$,
for $\ell=1,\ldots,L-1$, where $\sigma$ is the network non-linearity. 
\item The limiting NTK $\Theta_{\infty}^{(L)}$ is defined in terms of the
kernels $\Sigma_{\infty}^{(\ell)}$ and the kernels $\dot{\Sigma}_{\infty}^{(\ell)}=\mathbb{L}_{\Sigma_{\infty}^{(\ell-1)}}^{\dot{\sigma}}$:
\[
\Theta_{\infty}^{(L)}=\sum_{\ell=1}^{L}\Sigma_{\infty}^{(\ell)}\dot{\Sigma}_{\infty}^{(\ell+1)}\ldots\dot{\Sigma}_{\infty}^{(L)}.
\]
\end{enumerate}
The NTK leads to convergence guarantees for DNNs in the infinite-width
limit, and connect their generalization to that of kernel
methods \citep{jacot2018neural,exact_arora2019}.

\subsection{Gram Matrices}

For a finite dataset $x_{1},\ldots,x_{N}\in\mathbb{R}^{n_{0}}$ and
a fixed depth $L\geq1$, we denote by $\tilde{\Theta}\in\mathbb{R}^{Nn_{L}\times Nn_{L}}$
the Gram matrix of $x_{1},\ldots,x_{N}$ with respect to the limiting
NTK, defined by
\[
\tilde{\Theta}_{ik,jm}=\Theta_{\infty}^{\left(L\right)}\left(x_{i},x_{j}\right)\delta_{km}.
\]

It is block diagonal because different outputs $k\neq m$ are asymptotically
uncorrelated. 

Similarly, for any (scalar) kernel $\mathcal{K}^{\left(L\right)}$
(such as the limiting kernels $\Sigma_{\infty}^{\left(L\right)},\Lambda_{\infty}^{\left(L\right)},\Upsilon_{\infty}^{\left(L\right)},\Phi_{\infty}^{\left(L\right)},\Xi_{\infty}^{(L)}$
introduced later), we denote the Gram
matrix of the datapoints by $\tilde{\mathcal{K}}$.

\section{Main Theorems}

\subsection{Hessian as $I+S$}

Using the above setup, the Hessian $H$ of the loss $\mathcal{C}\circ F^{(L)}$
is the sum of two terms, with the entry $H_{p,p'}$ given by 
\[
H_{p,p'}=\mathcal{H}\mathcal{C}_{\vert f_{\theta}}(\partial_{\theta_{p}}F,\partial_{\theta_{p'}}F)+\mathcal{D}\mathcal{C}_{\vert f_{\theta}}(\partial_{\theta_{p},\theta_{p'}}F).
\]
For a finite dataset, the Hessian matrix $\mathcal{H}\left(C\circ Y^{(L)}\right)$
is equal to the sum of two matrices 
\begin{align*}
I=\left(\mathcal{D}Y^{(L)}\right)^{T}\mathcal{H}C\mathcal{D}Y^{(L)}\;\;\text{ and }\;\;S=\nabla C\cdot\mathcal{H}Y^{(L)}
\end{align*}
where $\mathcal{D}Y^{(L)}$ is a $Nn_{L}\times P$ matrix, $\mathcal{H}C$
is a $Nn_{L}\times Nn_{L}$ matrix and $\mathcal{H}Y^{(L)}$ is a
$P\times P\times Nn_{L}$ tensor to which we apply a scalar product
(denoted by $\cdot$) in its last dimension with the $Nn_{L}$ vector
$\nabla C$ to obtain a $P\times P$ matrix.

Our main contribution is the following theorem, which describes the
limiting moments $\mathrm{Tr}\left(H^{k}\right)$ in terms of the
moments of $I$ and $S$:
\begin{thm}
\label{thm:main-thm}For any loss $C$ with BGOSS and $\sigma\in C_{b}^{4}(\mathbb{R})$,
in the sequential limit $n_{1}\to\infty,\ldots,n_{L-1}\to\infty$,
we have for all $k\geq1$
\[
\mathrm{Tr}\left(H\left(t\right)^{k}\right)\approx\mathrm{Tr}\left(I\left(t\right)^{k}\right)+\mathrm{Tr}\left(S\left(t\right)^{k}\right).
\]
The limits of $\mathrm{Tr}\left(I\left(t\right)^{k}\right)$ and $\mathrm{Tr}\left(S\left(t\right)^{k}\right)$
can be expressed in terms of the NTK $\Theta_{\infty}^{\left(L\right)}$,
the kernels $\Upsilon_{\infty}^{\left(L\right)}, \Xi_{\infty}^{\left(L\right)}$ and the non-symmetric
kernels $\Phi_{\infty}^{\left(L\right)}$, $\Lambda_{\infty}^{\left(L\right)}$
defined in Appendix \ref{sec:The-Matrix-S}:
\begin{itemize}
\item The moments $\mathrm{Tr}\left(I\left(t\right)^{k}\right)$ converge
to the following limits (with the convention that $i_{k+1}=i_{1}$):
\[
\mathrm{Tr}\left(I\left(t\right)^{k}\right)\to\mathrm{Tr}\left(\left(\mathcal{H}C(Y\left(t\right))\tilde{\Theta}\right){}^{k}\right)=\frac{1}{N^{k}}\sum_{i_{1},...,i_{k}=1}^{N}\prod_{m=1}^{k}c''_{i_{m}}(f_{\theta(t)}(x_{i_{m}}))\Theta_{\infty}^{(L)}(x_{i_{m}},x_{i_{m+1}}).
\]
\item The first moment $\mathrm{Tr}\left(S\left(t\right)\right)$ converges
to the limit:
\[
\mathrm{Tr}\left(S\left(t\right)\right)=\left(G(t)\right)^{T}\nabla C(Y(t)).
\]
At initialization $(G(0),Y(0))$ form a Gaussian pair of $Nn_{L}$-vectors,
independent for differing output indices $k=1,...,n_{L}$ and with
covariance $\mathbb{E}\left[G_{ik}(0)G_{i'k'}(0)\right]=\delta_{kk'}\Xi_{\infty}^{(L)}(x_{i},x_{i'})$ and $\mathbb{E}\left[G_{ik}(0)Y_{i'k'}(0)\right]=\delta_{kk'}\Phi_{\infty}^{(L)}(x_{i},x_{i'})$
for the limiting kernel $\Xi_{\infty}^{(L)}(x,y)$ and non-symmetric kernel $\Phi_{\infty}^{(L)}(x,y)$.
During training, both vectors follow the differential equations
\begin{align*}
\partial_{t}G(t) & =-\tilde{\Lambda}\nabla C(Y(t))\\
\partial_{t}Y(t) & =-\tilde{\Theta}\nabla C(Y(t)).
\end{align*}
\item The second moment $\mathrm{Tr}\left(S\left(t\right)^{2}\right)$ converges
to the following limit defined in terms of the Gram matrix $\tilde{\Upsilon}$:
\[
\mathrm{Tr}\left(S^{2}\right)\to\left(\nabla C(Y(t))\right)^{T}\tilde{\Upsilon}\nabla C(Y(t))
\]
\item The higher moments $\mathrm{Tr}\left(S\left(t\right)^{k}\right)$
for $k\geq3$ vanish.
\end{itemize}
\end{thm}
\begin{proof}
The moments of $I$ and $S$ can be studied separately because the
moments of their sum is asymptotically equal to the sum of their moments
by Proposition \ref{prop:orthogonality_I_S} below. The limiting moments
of $I$ and $S$ are respectively described by Propositions \ref{prop:moments-I}
and \ref{prop:S_moments} below.
\end{proof}
In the case of a MSE loss $C(Y)=\frac{1}{2N}\left\Vert Y-Y^{*}\right\Vert ^{2}$,
the first and second derivatives take simple forms $\nabla C(Y)=\frac{1}{N}\left(Y-Y^{*}\right)$
and $\mathcal{H}C(Y)=\frac{1}{N}Id_{Nn_{L}}$ and the differential
equations can be solved to obtain more explicit formulae: 
\begin{cor}
\label{cor:mse-corr}For the MSE loss $C$ and $\sigma\in C_{b}^{4}(\mathbb{R})$,
in the limit $n_{1},...,n_{L-1}\to\infty$, we have uniformly over
$[0,T]$
\[
\mathrm{Tr}\left(H(t)^{k}\right)\to\frac{1}{N^{k}}\mathrm{Tr}\left(\tilde{\Theta}^{k}\right)+\mathrm{Tr}\left(S(t)^{k}\right)
\]
where 
\begin{align*}
\mathrm{Tr}\left(S(t)\right)\to & -\frac{1}{N}(Y^{*}-Y(0))^{T}\left(Id_{Nn_{L}}+e^{-t\tilde{\Theta}}\right)\tilde{\Theta}^{-1}\tilde{\Lambda}^{T}e^{-t\tilde{\Theta}}(Y^{*}-Y(0))\\
 & +\frac{1}{N}G(0)^{T}e^{-t\tilde{\Theta}}(Y^{*}-Y(0))\\
\mathrm{Tr}\left(S(t)^{2}\right)\to & \frac{1}{N^{2}}(Y^{*}-Y(0))^{T}e^{-t\tilde{\Theta}}\tilde{\Upsilon}e^{-t\tilde{\Theta}}(Y^{*}-Y(0))\\
\mathrm{Tr}\left(S(t)^{k}\right)\to & 0\,\,\,\,\,\text{when }k>2.
\end{align*}
In expectation we have:
\begin{align*}
\mathbb{E}\left[\mathrm{Tr}\left(S(t)\right)\right]\to & -\frac{1}{N}Tr\left(\text{\ensuremath{\left(Id_{Nn_{L}}+e^{-t\tilde{\Theta}}\right)\tilde{\Theta}^{-1}\tilde{\Lambda}^{T}e^{-t\tilde{\Theta}}}\ensuremath{\ensuremath{\left(\tilde{\Sigma}+Y^{*}Y^{*}{}^{T}\right)}}}\right)+\frac{1}{N}Tr\left(\text{\ensuremath{e^{-t\tilde{\Theta}}}\ensuremath{\ensuremath{\tilde{\Phi}^{T}}}}\right)\\
\mathbb{E}\left[\mathrm{Tr}\left(S(t)^{2}\right)\right]\to & \frac{1}{N^{2}}Tr\left(\text{\ensuremath{e^{-t\tilde{\Theta}}\tilde{\Upsilon}e^{-t\tilde{\Theta}}}\ensuremath{\ensuremath{\left(\tilde{\Sigma}+Y^{*}Y^{*}{}^{T}\right)}}}\right).
\end{align*}
\end{cor}
\begin{proof}
The moments of $I$ are constant because $\mathcal{H}C=\frac{1}{N}Id_{Nn_{L}}$
is constant. For the moments of $S$, we first solve the differential
equation for $Y(t)$:
\[
Y(t)=Y^{*}-e^{-t\tilde{\Theta}}(Y^{*}-Y(0)).
\]
Noting $Y(t)-Y(0)=-\tilde{\Theta}\int_{0}^{t}\nabla C(s)ds$, we have
\begin{align*}
G(t) & =G(0)-\tilde{\Lambda}\int_{0}^{t}\nabla C(s)ds\\
 & =G(0)+\tilde{\Lambda}\tilde{\Theta}^{-1}(Y(t)-Y(0))\\
 & =G(0)+\tilde{\Lambda}\tilde{\Theta}^{-1}\left(Id_{Nn_{L}}+e^{-t\tilde{\Theta}}\right)(Y^{*}-Y(0))
\end{align*}
The expectation of the first moment of $S$ then follows. 
\end{proof}
\begin{figure}
\centering
\!\!\!\!\!\!\!\!\!\!\includegraphics[scale=0.4]{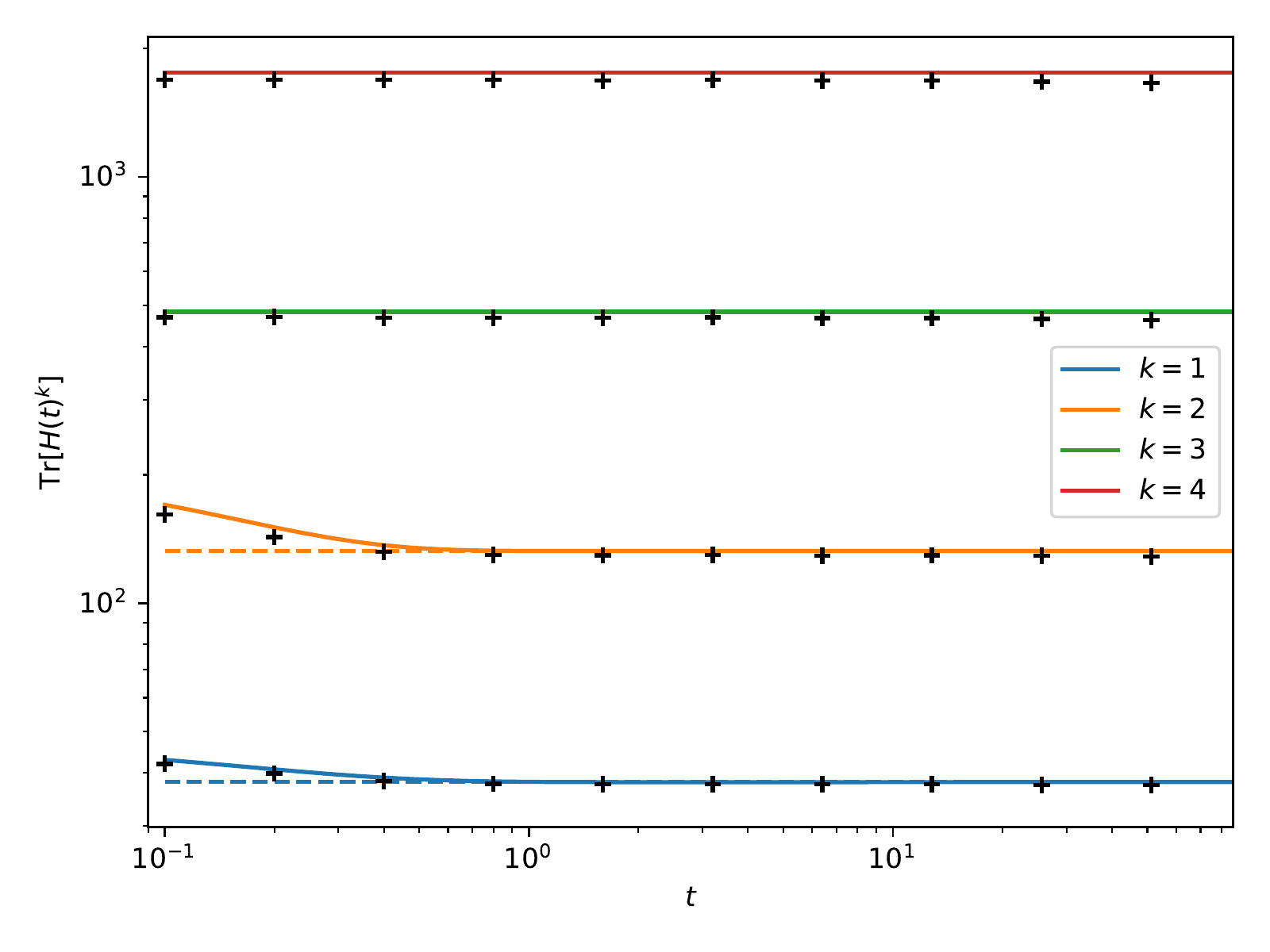}\includegraphics[scale=0.4]{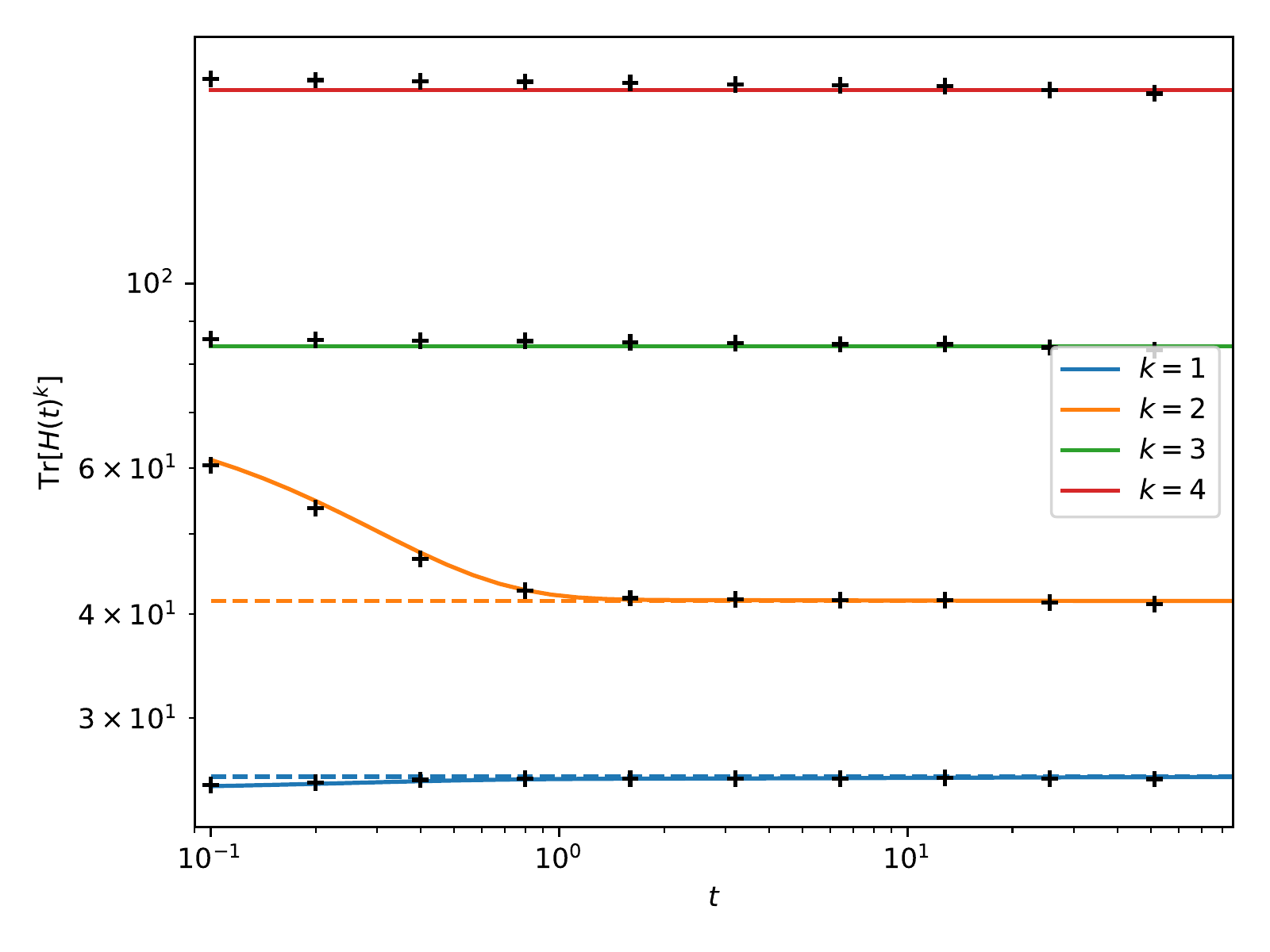}

\caption{\label{fig:moments}Comparison of the theoretical prediction of Corollary
\ref{cor:mse-corr} for the expectation of the first 4 moments (colored
lines) to the empirical average over 250 trials (black crosses) for
a rectangular network with two hidden layers of finite widths $n_{1}=n_{2}=5000$
($L=3$) with the smooth ReLU (left) and the normalized smooth ReLU
(right), for the MSE loss on scaled down 14x14 MNIST with $N=256$.
Only the first two moments are affected by $S$ at the beginning of
training.}
\end{figure}

\subsection{Mutual Orthogonality of $I$ and $S$}

A first key ingredient to prove Theorem \ref{thm:main-thm} is the
asymptotic mutual orthogonality of the matrices $I$ and $S$
\begin{prop*}[Proposition \ref{prop:orthogonality_I_S} in Appendix \ref{sec:Orthogonality-of-I-S}]
For any loss $C$ with BGOSS and $\sigma\in C_{b}^{4}(\mathbb{R})$,
we have uniformly over $[0,T]$
\[
\lim_{n_{L-1}\to\infty}\cdots\lim_{n_{1}\to\infty}\|IS\|_{F}=0.
\]
As a consequence $\lim_{n_{L-1}\to\infty}\cdots\lim_{n_{1}\to\infty}\mathrm{Tr}\left(\left[I+S\right]^{k}\right)-\left[\mathrm{Tr}\left(I^{k}\right)+\mathrm{Tr}\left(S^{k}\right)\right]=0$.
\end{prop*}
\begin{rem}
If two matrices $A$ and $B$ are mutualy orthogonal (i.e. $AB=0$)
the range of $A$ is contained in the nullspace of $B$ and vice versa.
The non-zero eigenvalues of the sum $A+B$ are therefore given by
the union of the non-zero eigenvalues of $A$ and $B$. Furthermore
the moments of $A$ and $B$ add up: $\mathrm{Tr}\left(\left[A+B\right]^{k}\right)=\mathrm{Tr}\left(A^{k}\right)+\mathrm{Tr}\left(B^{k}\right)$.
Proposition \ref{prop:orthogonality_I_S} shows that this is what
happens asymptotically for $I$ and $S$. 
\end{rem}
Note that both matrices $I$ and $S$ have large nullspaces: indeed
assuming a constant width $w=n_{1}=...=n_{L-1}$, we have $Rank(I)\leq Nn_{L}$
and $Rank(S)\leq2(L-1)wNn_{L}$ (see Appendix \ref{sec:The-Matrix-S}),
while the number of parameters $P$ scales as $w^{2}$ (when $L>2$).

Figure \ref{fig:sorthogonality} illustrates the mutual orthogonality
of $I$ and $S$. All numerical experiments are done for rectangular networks (when the width of the hidden layers are equal) and agree well with our predictions obtained in the sequential limit.

\subsection{Mean-field Limit}\label{sec:meanfield}
For a rectangular network with width $w$, if the output of the network is divided by $\sqrt{w}$ and the learning rate is multiplied by $w$ (to keep similar dynamics at initialization), the training dynamics changes and the NTK varies during training when $w$ goes to infinity. The new parametrization of the output changes the scaling of the two matrices:
\[
\mathcal{H}\left[ C\left(\frac{1}{\sqrt{w}}Y^{(L)}\right)\right] = \frac{1}{w} \left(\mathcal{D}Y^{(L)}\right)^T \mathcal{H}C \mathcal{D}Y^{(L)} + \frac{1}{\sqrt{w}}\nabla C \cdot \mathcal{H}Y^{(L)} = \frac{1}{w}I+\frac{1}{\sqrt{w}}S.
\]
The scaling of the learning rate essentially multiplies the whole Hessian by $w$. In this setting, the matrix $I$ is left unchanged while the matrix $S$ is multiplied by $\sqrt{w}$ (the $k$-th moment of $S$ is hence multiplied by $w^{\nicefrac{k}{2}}$). In particular, the two moments of the Hessian are dominated by the moments of $S$, and the higher moments of $S$ (and the operator norm of $S$) should not vanish. This suggests that the active regime may be characterised by the fact that $\|S\|_F \gg \|I\|_F $. Under the conjecture that Theorem \ref{thm:main-thm} holds for the infinite-width limit of rectangular networks, the asymptotic of the two first moments of $H$ is given by: 
\begin{align*}
\nicefrac{1}{\sqrt{w}} \mathrm{Tr}\left(H \right) &\to \mathcal{N}(0, \nabla C^T \tilde{\Xi} \nabla C)\\
\nicefrac{1}{w} \mathrm{Tr}\left(H^2 \right) &\to  \nabla C^T \tilde{\Upsilon} \nabla C,
\end{align*}
where for the MSE loss we have $\nabla C = -Y^*$.

\subsection{The matrix $S$}

The matrix $S=\nabla C\cdot\mathcal{H}Y^{(L)}$ is best understood
as a perturbation to $I$, which vanishes as the network converges
because $\nabla C\to0$. To calculate its moments, we note that
\[
\mathrm{Tr}\left(\nabla C\cdot\mathcal{H}Y^{(L)}\right)=\left(\sum_{p=1}^{P}\partial_{\theta_{p}^{2}}^{2}Y\right)^{T}\nabla C=G^{T}\nabla C,
\]
where the vector $G=\sum_{k=1}^{P}\partial_{\theta_{p}^{2}}^{2}Y\in\mathbb{R}^{Nn_{L}}$
is the evaluation of the function $g_{\theta}(x)=\sum_{k=1}^{P}\partial_{\theta_{p}^{2}}^{2}f_{\theta}(x)$
on the training set.

For the second moment we have

\[
\mathrm{Tr}\left(\left(\nabla C\cdot\mathcal{H}Y^{(L)}\right)^{2}\right)=\nabla C^{T}\left(\sum_{p,p'=1}^{P}\partial_{\theta_{p}\theta_{p'}}^{2}Y\left(\partial_{\theta_{p}\theta_{p'}}^{2}Y\right)^{T}\right)\nabla C=\nabla C^{T}\tilde{\Upsilon}\nabla C
\]
for $\tilde{\Upsilon}$ the Gram matrix of the kernel $\Upsilon^{(L)}(x,y)=\sum_{p,p'=1}^{P}\partial_{\theta_{p}\theta_{p'}}^{2}f_{\theta}(x)\left(\partial_{\theta_{p}\theta_{p'}}^{2}f_{\theta}(y)\right)^{T}$.

The following proposition desribes the limit of the function $g_{\theta}$
and the kernel $\Upsilon^{(L)}$ and the vanishing of the higher moments:
\begin{prop*}[Proposition \ref{prop:S_moments} in Appendix \ref{sec:The-Matrix-S}]
For any loss $C$ with BGOSS and $\sigma\in C_{b}^{4}(\mathbb{R})$,
the first two moments of $S$ take the form
\begin{align*}
\mathrm{Tr}\left(S(t)\right) & =G(t)^{T}\nabla C(t)\\
\mathrm{Tr}\left(S(t)^{2}\right) & =\nabla C(t)^{T}\tilde{\Upsilon}(t)\nabla C(t)
\end{align*}

- At initialization, $g_{\theta}$ and $f_{\theta}$ converge to a
(centered) Gaussian pair with covariances
\begin{align*}
\mathbb{E}[g_{\theta,k}(x)g_{\theta,k'}(x')] & =\delta_{kk'}\Xi_{\infty}^{(L)}(x,x')\\
\mathbb{E}[g_{\theta,k}(x)f_{\theta,k'}(x')] & =\delta_{kk'}\Phi_{\infty}^{(L)}(x,x')\\
\mathbb{E}[f_{\theta,k}(x)f_{\theta,k'}(x')] & =\delta_{kk'}\Sigma_{\infty}^{(L)}(x,x')
\end{align*}
 and during training $g_{\theta}$ evolves according to
\[
\partial_{t}g_{\theta,k}(x)=\sum_{i=1}^{N}\Lambda_{\infty}^{(L)}(x,x_{i})\partial_{ik}C(Y(t))_{\cdot}
\]

- Uniformly over any interval $[0,T]$, the kernel $\Upsilon^{(L)}$
has a deterministic and fixed limit $\lim_{n_{L-1}\to\infty}\cdots\lim_{n_{1}\to\infty}\Upsilon_{kk'}^{(L)}(x,x')=\delta_{kk'}\Upsilon_{\infty}^{(L)}(x,x')$
with limiting kernel:

\[
\Upsilon_{\infty}^{(L)}(x,x')=\sum_{\ell=1}^{L-1}\left(\Theta_{\infty}^{(\ell)}(x,x')^{2}\ddot{\Sigma}_{\infty}^{(\ell)}(x,x')+2\Theta_{\infty}^{(\ell)}(x,x')\dot{\Sigma}_{\infty}^{(\ell)}(x,x')\right)\dot{\Sigma}_{\infty}^{(\ell+1)}(x,x')\cdots\dot{\Sigma}_{\infty}^{(L-1)}(x,x').
\]

- The higher moment $k>2$ vanish: $\lim_{n_{L-1}\to\infty}\cdots\lim_{n_{1}\to\infty}\mathrm{Tr}\left(S^{k}\right)=0$.
\end{prop*}
This result has a number of consequences for infinitely wide networks:%

\begin{enumerate}
\item At initialization, the matrix $S$ has a finite Frobenius norm $\left\Vert S\right\Vert _{F}^{2}=\mathrm{Tr}\left(S^{2}\right)=\nabla C^{T}\tilde{\Upsilon}\nabla C$,
because $\Upsilon$ converges to a fixed limit. As the network converges,
the derivative of the cost goes to zero $\nabla C(t)\to0$ and so
does the Frobenius norm of $S$. %
\item In contrast the operator norm of $S$ vanishes already at initialization
(because for all even $k$, we have $\left\Vert S\right\Vert _{op}\leq\sqrt[k]{\mathrm{Tr}\left(S^{k}\right)}\to0$).
At initialization, the vanishing of $S$ in operator norm but not
in Frobenius norm can be explained by the matrix $S$ having a growing
number of eigenvalues of shrinking intensity as the width grows.
\item When it comes to the first moment of $S$, Proposition \ref{prop:S_moments}
shows that the spectrum of $S$ is in general not symmetric. For the
MSE loss the expectation of the first moment at initialization is
\[
\mathbb{E}\left[\mathrm{Tr}(S)\right]=\mathbb{E}\left[(Y-Y^{*})^{T}G\right]=\mathbb{E}\left[Y^{T}G\right]-\left(Y^{*}\right)^{T}\mathbb{E}\left[G\right]=\mathrm{Tr}\left(\tilde{\Phi}\right)-0
\]
which may be positive or negative depending on the choice of nonlinearity:
with a smooth ReLU, it is positive, while for the arc-tangent or the
normalized smooth ReLU, it can be negative (see Figure \ref{fig:moments}).\\
This is in contrast to the result obtained in \citep{Pennington2017,Geiger18}
for the shallow ReLU networks, taking the second derivative of the
ReLU to be zero. Under this assumption the spectrum of $S$ is symmetric:
if the eigenvalues are ordered from lowest to highest, $\lambda_{i}=-\lambda_{P-i}$
and $\mathrm{Tr}(S)=0$.%
\end{enumerate}

These observations suggest that $S$ has little influence on the shape
of the surface, especially towards the end of training, the matrix
$I$ however has an interesting structure.

\begin{figure}
\centering
\!\!
\begin{minipage}{.4\textwidth}\centering\includegraphics[scale=0.35]{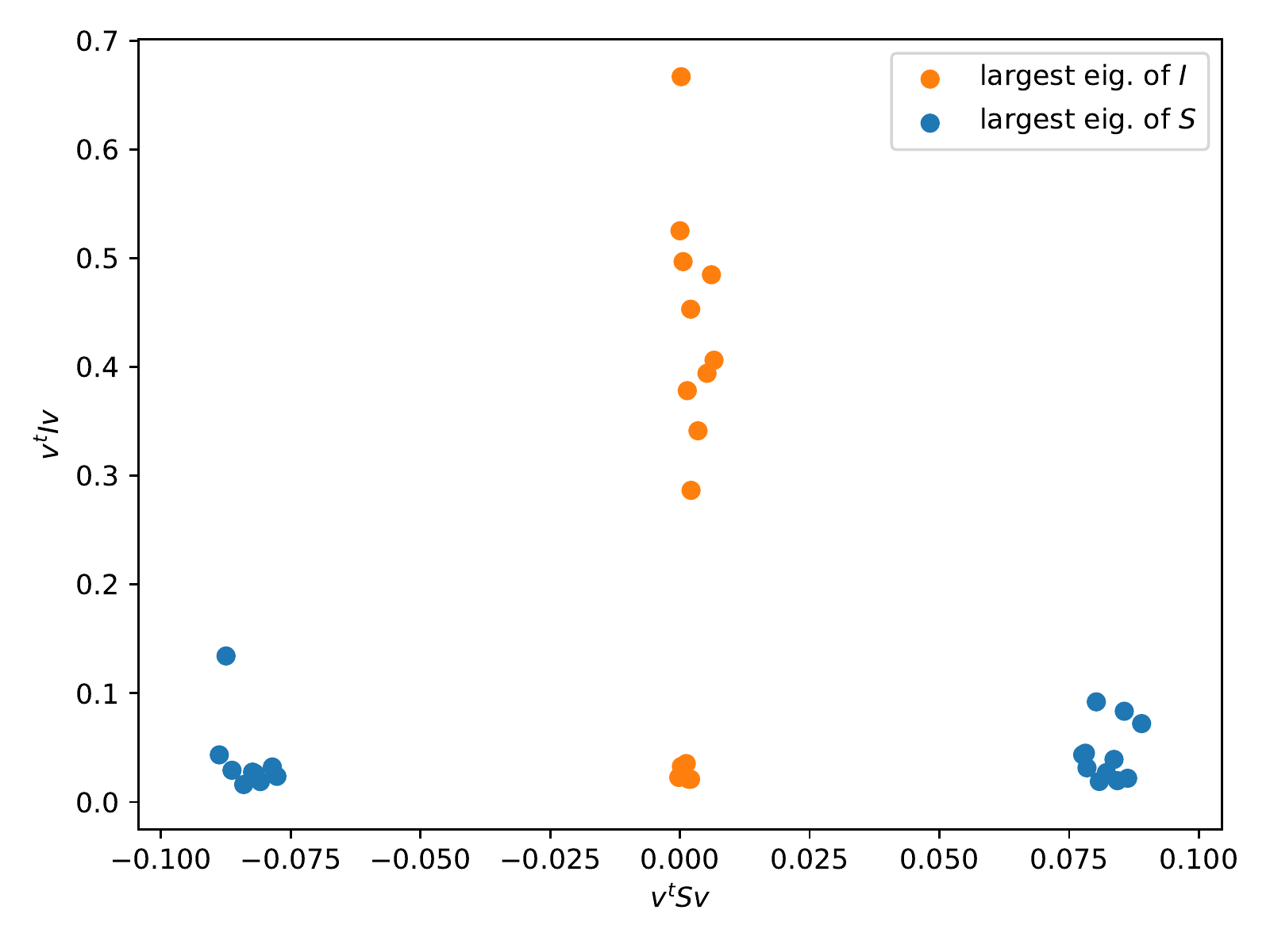}

\caption{\label{fig:sorthogonality}Illustration of the mutual orthogonality
of $I$ and $S$. For the 20 first eigenvectors of $I$ (blue) and
$S$ (orange), we plot the Rayleigh quotients $v^{T}Iv$ and $v^{T}Sv$
(with $L=3$, $n_{1}=n_{2}=1000$ and the normalized ReLU on 14x14
MNIST with $N=256$). We see that the directions where $I$ is large
are directions where $S$ is small and vice versa.}
\end{minipage}\;\;\;\;\;
\begin{minipage}{.6\textwidth}
\centering\includegraphics[scale=0.32]{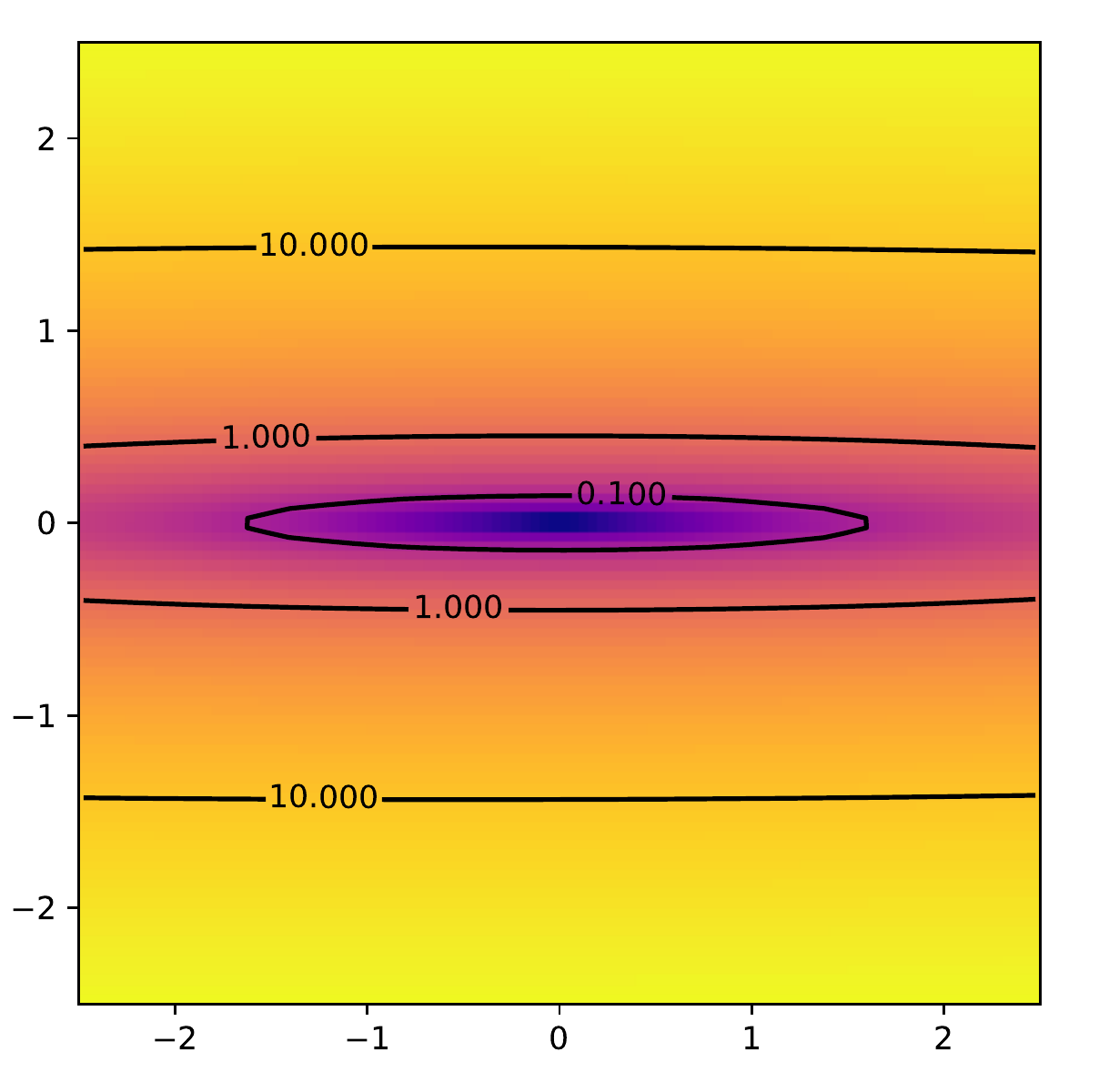}\includegraphics[scale=0.32]{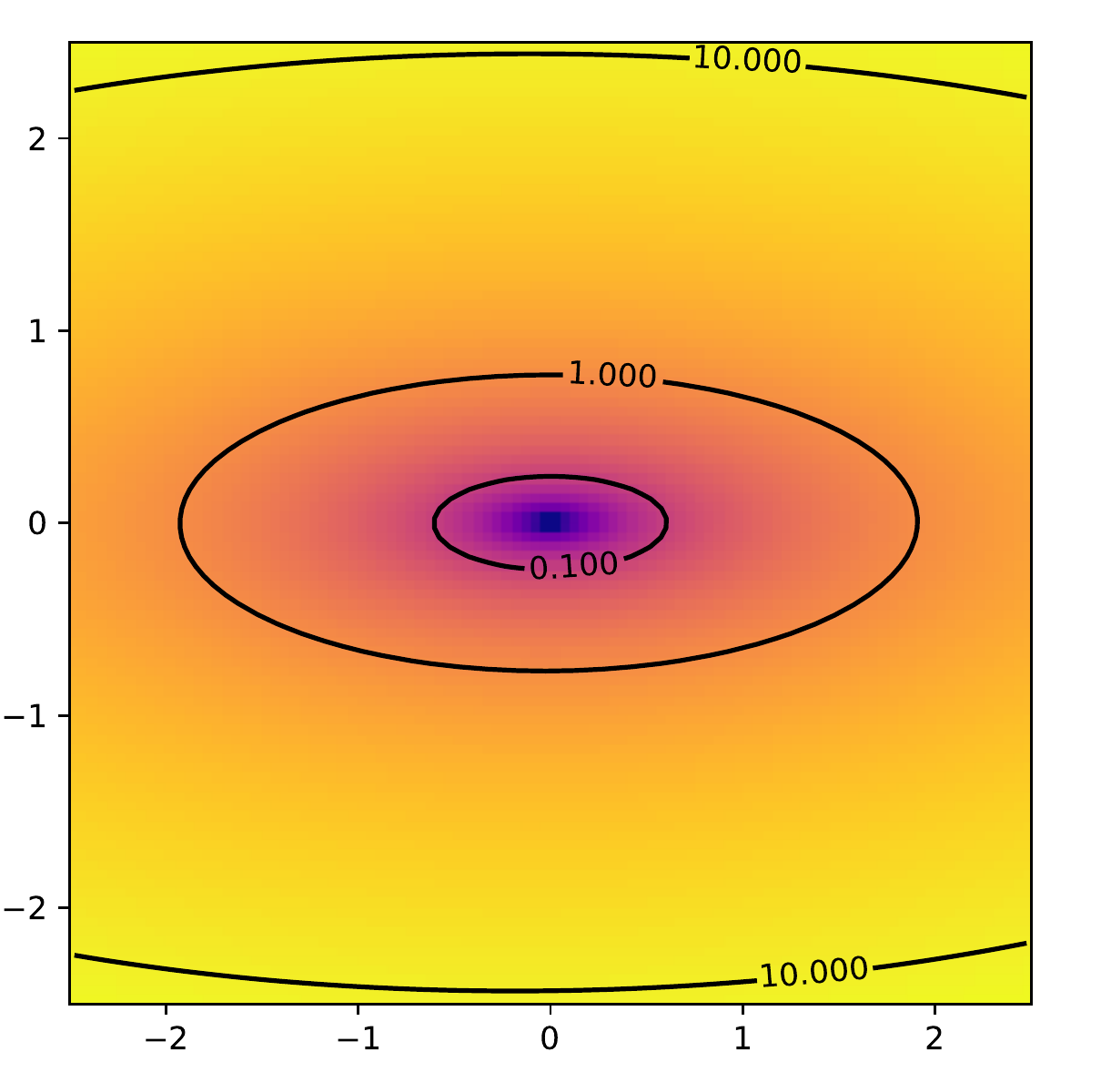}\caption{\label{fig:surface_narrow_valley}Plot of the loss surface around
a global minimum along the first (along the y coordinate) and fourth
(x coordinate) eigenvectors of $I$. The network has $L=4$, width
$n_{1}=n_{2}=n_{3}=1000$ for the smooth ReLU (left) and the normalized
smooth ReLU (right). The data is uniform on the unit disk. Normalizing
the non-linearity greatly reduces the narrow valley structure of the
loss thus speeding up training.}
\end{minipage}
\end{figure}

\subsection{The matrix $I$}

At a global minimizer $\theta^{*}$, the spectrum of $I$ describes
how the loss behaves around $\theta^{*}$. Along the eigenvectors
of the biggest eigenvalues of $I$, the loss increases rapidely, while
small eigenvalues correspond to flat directions. Numerically, it has
been observed that the matrix $I$ features a few dominating eigenvalues
and a bulk of small eigenvalues \citep{Singularity_Hessian_Sagun16,Sagun,GD_Tiny_subspace_GurAri_2018,Hessian_Three_Levels_Papyan_2019}.
This leads to a narrow valley structure of the loss around a minimum:
the biggest eigenvalues are the `cliffs' of the valley, i.e. the directions
along which the loss grows fastest, while the small eigenvalues form
the `flat directions'or the bottom of the valley.

Note that the rank of $I$ is bounded by $Nn_{L}$ and in the overparametrized
regime, when $Nn_{L}<P$, the matrix $I$ will have a large nullspace,
these are directions along which the value of the function on the
training set does not change. Note that in the overparametrized regime,
global minima are not isolated: they lie in a manifold of dimension
at least $P-Nn_{L}$ and the nullspace of $I$ is tangent to this
solution manifold.

The matrix $I$ is closely related to the NTK Gram matrix:
\[
\tilde{\Theta}=\mathcal{D}Y^{(L)}\left(\mathcal{D}Y^{(L)}\right)^{T}\text{ and }I=\left(\mathcal{D}Y^{(L)}\right)^{T}\mathcal{H}C\mathcal{D}Y^{(L)}.
\]

As a result, the limiting spectrum of the matrix $I$ can be directly
obtained from the NTK\footnote{This result was already obtained in \citep{Karakida2018}, but without
identifying the NTK explicitely and only at initialization.}
\begin{prop}
\label{prop:moments-I}For any loss $C$ with BGOSS and $\sigma\in C_{b}^{4}(\mathbb{R})$,
uniformly over any interval $[0,T]$, the moments $\mathrm{Tr}\left(I^{k}\right)$
converge to the following limit (with the convention that $i_{k+1}=i_{1}$):
\[
\lim_{n_{L-1}\to\infty}\negthickspace\cdots\negthickspace\lim_{n_{1}\to\infty}\mathrm{Tr}\left(I^{k}\right)=\mathrm{Tr}\left(\left(\mathcal{H}C(Y_{t})\tilde{\Theta}\right){}^{k}\right)=\frac{1}{N^{k}}\sum_{i_{1},...,i_{k}=1}^{N}\prod_{m=1}^{k}c''_{i_{m}}(f_{\theta(t)}(x_{i_{m}}))\Theta_{\infty}^{(L)}(x_{i_{m}},x_{i_{m+1}})
\]
\end{prop}
\begin{proof}
It follows from $\mathrm{Tr}\left(I^{k}\right)=\mathrm{Tr}\left(\left(\left(\mathcal{D}Y^{(L)}\right)^{T}\mathcal{H}C\mathcal{D}Y^{(L)}\right)^{k}\right)=\mathrm{Tr}\left(\left(\mathcal{H}C\tilde{\Theta}\right)^{k}\right)$
and the asymptotic of the NTK \citep{jacot2018neural}.
\end{proof}

\subsubsection{Mean-Square Error}

When the loss is the MSE, $\mathcal{H}C$ is equal to $\frac{1}{N}Id_{Nn_{L}}$.
As a result, $\tilde{\Theta}$ and $I$ have the same non-zero eigenvalues
up to a scaling of $\nicefrac{1}{N}$. Because the NTK is assymptotically
fixed, the spectrum of $I$ is also fixed in the limit.

The eigenvectors of the NTK Gram matrix are the kernel principal components
of the data. The biggest principal components are the directions in
function space which are most favorised by the NTK. This gives a functional
interpretation of the narrow valley structure in DNNs: the cliffs
of the valley are the biggest principal components, while the flat
directions are the smallest components.
\begin{rem}
As the depth $L$ of the network increases, one can observe two regimes
\citep{Chaos_Poole2016,Freeze_Chaos_Jacot2019}: Order/Freeze where
the NTK converges to a constant and Chaos where the NTK converges
to a Kronecker delta. In the Order/Freeze the $Nn_{L}\times Nn_{L}$
Gram matrix approaches a block diagonal matrix with $n_{L}$ constant
blocks, and as a result $n_{L}$ eigenvalues of $I$ dominate the
other ones, corresponding to constant directions along each outputs
(this is in line with the observations of \citep{Hessian_Three_Levels_Papyan_2019}).
This leads to a narrow valley for the loss and slows down training.
In contrast, in the Chaos regime, the NTK Gram matrix approaches a
scaled identity matrix, and the spectrum of $I$ should hence concentrate
around a positive value, hence speeding up training. Figure \ref{fig:surface_narrow_valley}
illustrates this phenomenon: with the smooth ReLU we observe a narrow
valley, while with the normalized smooth ReLU (which lies in the Chaos
according to \citep{Freeze_Chaos_Jacot2019}) the narrowness of the
loss is reduced. A similar phenomenon may explain why normalization
helps smoothing the loss surface and speed up training \citep{HowDoesBN_Santurkar2018,Hessian_BN_outliers_ghorbani2019}.

\end{rem}

\subsubsection{Cross-Entropy Loss}

For a binary cross-entropy loss with labels $Y^{*}\in\{-1,+1\}^{N}$
\[
C(Y)=\frac{1}{N}\sum_{i=1}^{N}log\left(1+e^{-Y_{i}^{*}Y_{i}}\right),
\]
$\mathcal{H}C$ is a diagonal matrix whose entries depend on $Y$
(but not on $Y^{*}$):
\[
\mathcal{H}_{ii}C(Y)=\frac{1}{N}\frac{1}{1+e^{-Y_{i}}+e^{Y_{i}}}.
\]
The eigenvectors of $I$ then correspond to the weighted kernel principal
component of the data. The positive weights $\frac{1}{1+e^{-Y_{i}}+e^{Y_{i}}}$
approach $\nicefrac{1}{3}$ as $Y_{i}$ goes to $0$, i.e. when it
is close to the decision boundary from one class to the other, and
as $Y_{i}\to\pm\infty$ the weight go to zero. The weights evolve
in time through $Y_{i}$, the spectrum of $I$ is therefore not asymptotically
fixed as in the MSE case, but the functional interpretation of the
spectrum in terms of the kernel principal components remains.

\section{Conclusion}

We have given an explicit formula for the limiting moments of the
Hessian of DNNs throughout training. We have used the common decomposition
of the Hessian in two terms $I$ and $S$ and have shown that the
two terms are asymptotically mutually orthogonal, such that they can
be studied separately. 

The matrix $S$ vanishes in Frobenius norm as the network converges
and has vanishing operator norm throughout training. The matrix $I$
is arguably the most important as it describes the narrow valley structure
of the loss around a global minimum. The eigendecomposition of $I$
is related to the (weighted) kernel principal components of the data
w.r.t. the NTK.

\section*{Acknowledgements}
Cl\'ement Hongler acknowledges support from the ERC SG CONSTAMIS grant, the NCCR SwissMAP grant, the NSF DMS-1106588 grant, the Minerva Foundation, the Blavatnik Family Foundation, and the Latsis foundation. 

\bibliographystyle{iclr2020_conference}
\bibliography{./../main}

\appendix

\section{Proofs\label{sec:Proofs}}

For the proofs of the theorems and propositions presented in the main
text, we reformulate the setup of \citep{jacot2018neural}. For a fixed
training set $x_{1},...,x_{N}$, we consider a (possibly random) time-varying
training direction $D(t)\in\mathbb{R}^{Nn_{L}}$ which describes how
each of the outputs must be modified. In the case of gradient descent
on a cost $C(Y)$, the training direction is $D(t)=\nabla C(Y(t))$.
The parameters are updated according to the differential equation
\[
\partial_{t}\theta(t)=\left(\partial_{\theta}Y(t)\right)^{T}D(t).
\]
Under the condition that $\int_{0}^{T}\left\Vert D(t)\right\Vert _{2}dt$
is stochastically bounded as the width of the network goes to infinity,
the NTK $\Theta^{(L)}$ converges to its fixed limit uniformly over
$[0,T]$. 

The reason we consider a general training direction (and not only
a gradient of a loss) is that we can split a network in two at a layer
$\ell$ and the training of the smaller network will be according
to the training direction $D_{i}^{(\ell)}(t)$ given by
\[
D_{i}^{(\ell)}(t)=diag\left(\dot{\sigma}\left(\alpha^{(\ell)}(x_{i})\right)\right)\left(\frac{1}{\sqrt{n_{\ell}}}W^{(\ell)}\right)^{T}...diag\left(\dot{\sigma}\left(\alpha^{(L-1)}(x_{i})\right)\right)\left(\frac{1}{\sqrt{n_{L-1}}}W^{(L-1)}\right)^{T}D_{i}(t)
\]
because the derivatives $\dot{\sigma}$ are bounded and by Lemma 1
of the Appendix of \citep{jacot2018neural}, this training direction
satisfies the constraints even though it is not the gradient of a
loss. As a consequence, as $n_{1}\to\infty,...,n_{\ell-1}\to\infty$
the NTK of the smaller network $\Theta^{(\ell)}$ also converges to
its limit uniformly over $[0,T]$. As we let $n_{\ell}\to\infty$
the pre-activations $\tilde{\alpha}_{i}^{(\ell)}$ and weights $W_{ij}^{(\ell)}$
move at a rate of $\nicefrac{1}{\sqrt{n_{\ell}}}$. We will use this
rate of change to prove that other types of kernels are constant during
training.

When a network is trained with gradient descent on a loss $C$ with
BGOSS, the integral $\int_{0}^{T}\left\Vert D(t)\right\Vert _{2}dt$
is stochastically bounded. Because the loss is decreasing during training,
the outputs $Y(t)$ lie in the sublevel set $U_{C(Y(0))}$ for all
times $t$. The norm of the gradient is hence bounded for all times
$t$. Because the distribution of $Y(0)$ converges to a multivariate
Gaussian, $b(C(Y(0)))$ is stochastically bounded as the width grows,
where $b(a)$ is a bound on the norm of the gradient on $U_{a}$.
We then have the bound $\int_{0}^{T}\left\Vert D(t)\right\Vert _{2}dt\leq Tb(C(Y(0)))$
which is itself stochastically bounded.

For the binary and softmax cross-entropy losses the gradient is uniformly
bounded:
\begin{prop}
\label{prop:bounded-gradient-cross-entropy}For the binary cross-entropy
loss $C$ and any $Y\in\mathbb{R}^{N}$, $\left\Vert \nabla C(Y)\right\Vert _{2}\leq\frac{1}{\sqrt{N}}$.

For the softmax cross-entropy loss C on $c\in\mathbb{N}$ classes
and any $Y\in\mathbb{R}^{Nc}$, $\left\Vert \nabla C(Y)\right\Vert _{2}\leq\frac{\sqrt{2c}}{\sqrt{N}}$.
\end{prop}
\begin{proof}
The binary cross-entropy loss with labels $Y^{*}\in\left\{ 0,1\right\} ^{N}$
is
\[
C(Y)=-\frac{1}{N}\sum_{i=1}^{N}\log\frac{e^{Y_{i}Y_{i}^{*}}}{1+e^{Y_{i}}}=\frac{1}{N}\sum_{i=1}^{N}\log\left(1+e^{Y_{i}}\right)-Y_{i}Y_{i}^{*}
\]
 and the gradient at an input $i$ is
\[
\mathcal{\partial}_{i}C(Y)=\frac{1}{N}\frac{e^{Y_{i}}-Y_{i}^{*}(1+e^{Y_{i}})}{1+e^{Y_{i}}}
\]
which is bounded in absolute value by $\frac{1}{N}$ for both $Y_{i}^{*}=0,1$
such that $\left\Vert \nabla C(Y)\right\Vert _{2}\leq\frac{1}{\sqrt{N}}$.

The softmax cross-entropy loss over $c$ classes with labels $Y^{*}\in\left\{ 1,\ldots,c\right\} ^{N}$
is defined by
\[
C(Y)=-\frac{1}{N}\sum_{i=1}^{N}\log\frac{e^{Y_{iY_{i}^{*}}}}{\sum_{k=1}^{c}e^{Y_{ik}}}=\frac{1}{N}\sum_{i=1}^{N}\log\left(\sum_{k=1}^{c}e^{Y_{ik}}\right)-Y_{iY_{i}^{*}}.
\]
The gradient is at an input $i$ and output class $m$ is 
\[
\mathcal{\partial}_{im}C(Y)=\frac{1}{N}\left(\frac{e^{Y_{im}}}{\sum_{k=1}^{c}e^{Y_{ik}}}-\delta_{Y_{i}^{*}m}\right)
\]
which is bounded in absolute value by $\frac{2}{N}$ such that $\left\Vert \nabla C(Y)\right\Vert _{2}\leq\frac{\sqrt{2c}}{\sqrt{N}}$.
\end{proof}

\section{Preliminaries}

To study the moments of the matrix $S$, we first have to show that
two tensors vanish as $n_{1},...,n_{L-1}\to\infty$:
\begin{align*}
\Omega_{k_{0},k_{1},k_{2}}^{(L)}(x_{0},x_{1},x_{2}) & =\left(\mathcal{\nabla}f_{\theta,k_{0}}(x_{0})\right)^{T}\mathcal{H}f_{\theta,k_{1}}(x_{1})\mathcal{\nabla}f_{\theta,k_{2}}(x_{2})\\
\Gamma_{k_{0},k_{1},k_{2},k_{3}}^{(L)}(x_{0},x_{1},x_{2},x_{4}) & =\left(\mathcal{\nabla}f_{\theta,k_{0}}(x_{0})\right)^{T}\mathcal{H}f_{\theta,k_{1}}(x_{1})\mathcal{H}f_{\theta,k_{2}}(x_{2})\mathcal{\nabla}f_{\theta,k_{3}}(x_{3}).
\end{align*}
We study these tensors recursively, for this, we need a recursive
definition for the first derivatives $\partial_{\theta_{p}}f_{\theta,k}(x)$
and second derivatives $\partial_{\theta_{p}\theta_{p'}}^{2}f_{\theta,k}(x)$.
The value of these derivatives depend on the layer $\ell$ the parameters
$\theta_{p}$ and $\theta_{p'}$ belong to, and on whether they are
connection weights $W_{mk}^{(\ell)}$ or biases $b_{k}^{(\ell)}$.
The derivatives with respect to the parameters of the last layer are
\begin{align*}
\partial_{W_{mk}^{(L-1)}}f_{\theta,k'}(x) & =\frac{1}{\sqrt{n_{L-1}}}\alpha_{m}^{(L-1)}(x)\delta_{kk'}\\
\partial_{b_{k}^{(L-1)}}f_{\theta,k'}(x) & =\beta^{2}\delta_{kk'}
\end{align*}
for parameters $\theta_{p}$ which belong to the lower layers the
derivatives can be defined recursively by
\[
\partial_{\theta_{p}}f_{\theta,k}(x)=\frac{1}{\sqrt{n_{L-1}}}\sum_{m=1}^{n_{L-1}}\partial_{\theta_{p}}\tilde{\alpha}_{m}^{(L-1)}(x)\dot{\sigma}\left(\tilde{\alpha}_{m}^{(L-1)}(x)\right)W_{mk}^{(L-1)}.
\]

For the second derivatives, we first note that if either of the parameters
$\theta_{p}$ or $\theta_{p'}$ are bias of the last layer, or if
they are both connection weights of the last layer, then $\partial_{\theta_{p}\theta_{p'}}^{2}f_{\theta,k}(x)=0$.
Two cases are left: when one parameter is a connection weight of the
last layer and the others belong to the lower layers, and when both
belong to the lower layers. Both cases can be defined recursively
in terms of the first and second derivatives of $\tilde{\alpha}_{m}^{(L-1)}$:
\begin{align*}
\partial_{\theta_{p}W_{mk}^{(L)}}^{2}f_{\theta,k'}(x) & =\frac{1}{\sqrt{n_{L-1}}}\partial_{\theta_{p}}\tilde{\alpha}_{m}^{(L-1)}(x)\dot{\sigma}\left(\tilde{\alpha}_{m}^{(L-1)}(x)\right)\delta_{kk'}\\
\partial_{\theta_{p}\theta_{p'}}^{2}f_{\theta,k'}(x) & =\frac{1}{\sqrt{n_{L-1}}}\sum_{m=1}^{n_{L-1}}\partial_{\theta_{p}\theta_{p'}}^{2}\tilde{\alpha}_{m}^{(L-1)}(x)\dot{\sigma}\left(\tilde{\alpha}_{m}^{(L-1)}(x)\right)W_{mk}^{(L-1)}\\
 & +\frac{1}{\sqrt{n_{L-1}}}\sum_{m=1}^{n_{L-1}}\partial_{\theta_{p}}\tilde{\alpha}_{m}^{(L-1)}(x)\partial_{\theta_{p'}}\tilde{\alpha}_{m}^{(L-1)}(x)\ddot{\sigma}\left(\tilde{\alpha}_{m}^{(L-1)}(x)\right)W_{mk}^{(L-1)}.
\end{align*}

Using these recursive definitions, the tensors $\Omega^{(L+1)}$ and
$\Gamma^{(L+1)}$ are given in terms of $\Theta^{(L)}$,$\Omega^{(L)}$
and $\Gamma^{(L)}$, in the same manner that the NTK $\Theta^{(L+1)}$
is defined recursively in terms of $\Theta^{(L)}$ in \citep{jacot2018neural}.
\begin{lem}
\label{lem:omega-vanishes}For any loss $C$ with BGOSS and $\sigma\in C_{b}^{4}(\mathbb{R})$,
we have uniformly over $[0,T]$
\[
\lim_{n_{L-1}\to\infty}\cdots\lim_{n_{1}\to\infty}\Omega_{k_{0},k_{1},k_{2}}^{(L)}(x_{0},x_{1},x_{2})=0
\]
\end{lem}
\begin{proof}
The proof is done by induction. When $L=1$ the second derivatives
$\partial_{\theta_{p}\theta_{p'}}^{2}f_{\theta,k}(x)=0$ and $\Omega_{k_{0},k_{1},k_{2}}^{(L)}(x_{0},x_{1},x_{2})=0$.

For the induction step, we write $\Omega_{k_{0},k_{1},k_{2}}^{(\ell+1)}(x_{0},x_{1},x_{2})$
recursively as 
\begin{align*}
 & n_{\ell}^{-\nicefrac{3}{2}}\sum_{m_{0},m_{1},m_{2}}\Theta_{m_{0},m_{1}}^{(\ell)}(x_{0},x_{1})\Theta_{m_{1},m_{2}}^{(\ell)}(x_{1},x_{2})\dot{\sigma}(\tilde{\alpha}_{m_{0}}^{(\ell)}(x_{0}))\ddot{\sigma}(\tilde{\alpha}_{m_{1}}^{(\ell)}(x_{1}))\dot{\sigma}(\tilde{\alpha}_{m_{2}}^{(\ell)}(x_{2}))W_{m_{0}k_{0}}^{(\ell)}W_{m_{1}k_{1}}^{(\ell)}W_{m_{2}k_{2}}^{(\ell)}\\
 & +n_{\ell}^{-\nicefrac{3}{2}}\sum_{m_{0},m_{1},m_{2}}\Omega_{m_{0},m_{1},m_{2}}^{(\ell)}(x_{0},x_{1},x_{2})\dot{\sigma}(\tilde{\alpha}_{m_{0}}^{(\ell)}(x_{0}))\dot{\sigma}(\tilde{\alpha}_{m_{1}}^{(\ell)}(x_{1}))\dot{\sigma}(\tilde{\alpha}_{m_{2}}^{(\ell)}(x_{2}))W_{m_{0}k_{0}}^{(\ell)}W_{m_{1}k_{1}}^{(\ell)}W_{m_{2}k_{2}}^{(\ell)}\\
 & +n_{\ell}^{-\nicefrac{3}{2}}\sum_{m_{0},m_{1}}\Theta_{m_{0},m_{1}}^{(\ell)}(x_{0},x_{1})\dot{\sigma}(\tilde{\alpha}_{m_{0}}^{(\ell)}(x_{0}))\dot{\sigma}(\tilde{\alpha}_{m_{1}}^{(\ell)}(x_{1}))\sigma(\tilde{\alpha}_{m_{1}}^{(\ell)}(x_{2}))W_{m_{0}k_{0}}^{(\ell)}\delta_{k_{1}k_{2}}\\
 & +n_{\ell}^{-\nicefrac{3}{2}}\sum_{m_{1},m_{2}}\Theta_{m_{1},m_{2}}^{(\ell)}(x_{1},x_{2})\sigma(\tilde{\alpha}_{m_{1}}^{(\ell)}(x_{0}))\dot{\sigma}(\tilde{\alpha}_{m_{1}}^{(\ell)}(x_{1}))\dot{\sigma}(\tilde{\alpha}_{m_{2}}^{(\ell)}(x_{2}))\delta_{k_{0}k_{1}}W_{m_{2}k_{2}}^{(\ell)}.
\end{align*}

As $n_{1},...,n_{\ell-1}\to\infty$ and for any times $t<T$, the
NTK $\Theta^{(\ell)}$ converges to its limit while $\Omega^{(\ell)}$
vanishes. The second summand hence vanishes and the others converge
to 
\begin{align*}
 & n_{\ell}^{-\nicefrac{3}{2}}\sum_{m}\Theta_{\infty}^{(\ell)}(x_{0},x_{1})\Theta_{\infty}^{(\ell)}(x_{1},x_{2})\dot{\sigma}(\tilde{\alpha}_{m}^{(\ell)}(x_{0}))\ddot{\sigma}(\tilde{\alpha}_{m}^{(\ell)}(x_{1}))\dot{\sigma}(\tilde{\alpha}_{m}^{(\ell)}(x_{2}))W_{mk_{0}}^{(\ell)}W_{mk_{1}}^{(\ell)}W_{mk_{2}}^{(\ell)}\\
 & +n_{\ell}^{-\nicefrac{3}{2}}\sum_{m}\Theta_{\infty}^{(\ell)}(x_{0},x_{1})\dot{\sigma}(\tilde{\alpha}_{m}^{(\ell)}(x_{0}))\dot{\sigma}(\tilde{\alpha}_{m}^{(\ell)}(x_{1}))\sigma(\tilde{\alpha}_{m}^{(\ell)}(x_{2}))W_{mk_{0}}^{(\ell)}\delta_{k_{1}k_{2}}\\
 & +n_{\ell}^{-\nicefrac{3}{2}}\sum_{m}\Theta_{\infty}^{(\ell)}(x_{1},x_{2})\sigma(\tilde{\alpha}_{m}^{(\ell)}(x_{0}))\dot{\sigma}(\tilde{\alpha}_{m}^{(\ell)}(x_{1}))\dot{\sigma}(\tilde{\alpha}_{m}^{(\ell)}(x_{2}))\delta_{k_{0}k_{1}}W_{mk_{2}}^{(\ell)}.
\end{align*}
At initialization, all terms vanish as $n_{\ell}\to\infty$ because
all summands are independent with zero mean and finite variance: in
the $n_{1}\to\infty,\ldots,n_{\ell-1}\to\infty$ limit, the $\tilde{\alpha}_{m}^{(\ell)}(x)$
are independent for different $m$, see \citep{jacot2018neural}. During
training, the weights $W^{(\ell)}$ and preactivations $\tilde{\alpha}^{(\ell)}$
move at a rate of $\nicefrac{1}{\sqrt{n_{\ell}}}$ (see the proof
of convergence of the NTK in \citep{jacot2018neural}). Since $\dot{\sigma}$
is Lipschitz, we obtain that the motion during training of each of
the sums is of order $n_{\ell}^{-\nicefrac{3}{2}+\nicefrac{1}{2}}=n_{\ell}^{-1}$.
As a result, uniformly over times $t\in\left[0,T\right]$, all the
sums vanish.
\end{proof}
Similarily, we have
\begin{lem}
\label{lem:Gamma_vanish}For any loss $C$ with BGOSS and $\sigma\in C_{b}^{4}(\mathbb{R})$,
we have uniformly over $[0,T]$
\[
\lim_{n_{L-1}\to\infty}\cdots\lim_{n_{1}\to\infty}\Gamma_{k_{0},k_{1},k_{2},k_{3}}^{(L)}(x_{0},x_{1},x_{2},x_{3})=0
\]
\end{lem}
\begin{proof}
The proof is done by induction. When $L=1$ the hessian $\mathcal{H}F^{(1)}=0$,
such that $\Gamma_{k_{0},k_{1},k_{2},k_{3}}^{(L)}(x_{0},x_{1},x_{2},x_{3})=0$.

For the induction step, $\Gamma^{(\ell+1)}$ can be defined recursively:

\begin{align*}
 & \Gamma_{k_{0},k_{1},k_{2},k_{3}}^{(L+1)}(x_{0},x_{1},x_{2},x_{3})\\
 & \begin{aligned}\begin{aligned}=n_{L}^{-2}\sum_{m_{0},m_{1},m_{2},m_{3}}\Gamma_{m_{0},m_{1},m_{2},m_{3}}^{(L)}(x_{0},x_{1},x_{2},x_{3})\dot{\sigma}(\alpha_{m_{0}}^{(L)}(x_{0}))\dot{\sigma}(\alpha_{m_{1}}^{(L)}(x_{1}))\dot{\sigma}(\alpha_{m_{2}}^{(L)}(x_{2}))\dot{\sigma}(\alpha_{m_{3}}^{(L)}(x_{3}))\\
W_{m_{0}k_{0}}^{(L)}W_{m_{1}k_{1}}^{(L)}W_{m_{2}k_{2}}^{(L)}W_{m_{3}k_{3}}^{(L)}
\end{aligned}
\end{aligned}
\\
 & \begin{aligned}+n_{L}^{-2}\sum_{m_{0},m_{1},m_{2},m_{3}}\Theta_{m_{0},m_{1}}^{(L)}(x_{0},x_{1})\Omega_{m_{1},m_{2},m_{3}}^{(L)}(x_{1},x_{2},x_{3})\dot{\sigma}(\alpha_{m_{0}}^{(L)}(x_{0}))\ddot{\sigma}(\alpha_{m_{1}}^{(L)}(x_{1}))\\
\dot{\sigma}(\alpha_{m_{2}}^{(L)}(x_{2}))\dot{\sigma}(\alpha_{m_{3}}^{(L)}(x_{3}))W_{m_{0}k_{0}}^{(L)}W_{m_{1}k_{1}}^{(L)}W_{m_{2}k_{2}}^{(L)}W_{m_{3}k_{3}}^{(L)}
\end{aligned}
\\
 & \begin{aligned}+n_{L}^{-2}\sum_{m_{0},m_{1},m_{2},m_{3}}\Omega_{m_{0},m_{1},m_{2}}^{(L)}(x_{0},x_{1},x_{2})\Theta_{m_{2},m_{3}}^{(L)}(x_{2},x_{3})\dot{\sigma}(\alpha_{m_{0}}^{(L)}(x_{0}))\dot{\sigma}(\alpha_{m_{1}}^{(L)}(x_{1}))\\
\ddot{\sigma}(\alpha_{m_{2}}^{(L)}(x_{2}))\dot{\sigma}(\alpha_{m_{3}}^{(L)}(x_{3}))W_{m_{0}k_{0}}^{(L)}W_{m_{1}k_{1}}^{(L)}W_{m_{2}k_{2}}^{(L)}W_{m_{3}k_{3}}^{(L)}
\end{aligned}
\\
 & \begin{aligned}+n_{L}^{-2}\sum_{m_{0},m_{1},m_{2},m_{3}}\Theta_{m_{0},m_{1}}^{(L)}(x_{0},x_{1})\Theta_{m_{1},m_{2}}^{(L)}(x_{1},x_{2})\Theta_{m_{2},m_{3}}^{(L)}(x_{2},x_{3})\dot{\sigma}(\alpha_{m_{0}}^{(L)}(x_{0}))\ddot{\sigma}(\alpha_{m_{1}}^{(L)}(x_{1}))\\
\ddot{\sigma}(\alpha_{m_{2}}^{(L)}(x_{2}))\dot{\sigma}(\alpha_{m_{3}}^{(L)}(x_{3}))W_{m_{0}k_{0}}^{(L)}W_{m_{1}k_{1}}^{(L)}W_{m_{2}k_{2}}^{(L)}W_{m_{3}k_{3}}^{(L)}
\end{aligned}
\\
 & \begin{aligned}+n_{L}^{-2}\sum_{m_{1},m_{2},m_{3}}\Omega_{m_{1},m_{2},m_{3}}^{(L)}(x_{1},x_{2},x_{3})\sigma(\alpha_{m_{1}}^{(L)}(x_{0}))\dot{\sigma}(\alpha_{m_{1}}^{(L)}(x_{1}))\dot{\sigma}(\alpha_{m_{2}}^{(L)}(x_{2}))\dot{\sigma}(\alpha_{m_{3}}^{(L)}(x_{3}))\\
\delta_{k_{0}k_{1}}W_{m_{2}k_{2}}^{(L)}W_{m_{3}k_{3}}^{(L)}
\end{aligned}
\\
 & \begin{aligned}+n_{L}^{-2}\sum_{m_{1},m_{2},m_{3}}\Theta_{m_{1},m_{2}}^{(L)}(x_{1},x_{2})\Theta_{m_{2},m_{3}}^{(L)}(x_{2},x_{3})\sigma(\alpha_{m_{1}}^{(L)}(x_{0}))\dot{\sigma}(\alpha_{m_{1}}^{(L)}(x_{1}))\ddot{\sigma}(\alpha_{m_{2}}^{(L)}(x_{2}))\dot{\sigma}(\alpha_{m_{3}}^{(L)}(x_{3}))\\
\delta_{k_{0}k_{1}}W_{m_{2}k_{2}}^{(L)}W_{m_{3}k_{3}}^{(L)}
\end{aligned}
\\
 & \begin{aligned}+n_{L}^{-2}\sum_{m_{0},m_{1},m_{2}}\Omega_{m_{0},m_{1},m_{2}}^{(L)}(x_{0},x_{1},x_{2})\dot{\sigma}(\alpha_{m_{0}}^{(L)}(x_{0}))\dot{\sigma}(\alpha_{m_{1}}^{(L)}(x_{1}))\dot{\sigma}(\alpha_{m_{2}}^{(L)}(x_{2}))\sigma(\alpha_{m_{2}}^{(L)}(x_{3}))\\
W_{m_{0}k_{0}}^{(L)}W_{m_{1}k_{1}}^{(L)}\delta_{k_{2}k_{3}}
\end{aligned}
\\
 & \begin{aligned}+n_{L}^{-2}\sum_{m_{0},m_{1},m_{2}}\Theta_{m_{0},m_{1}}^{(L)}(x_{0},x_{1})\Theta_{m_{1},m_{2}}^{(L)}(x_{1},x_{2})\dot{\sigma}(\alpha_{m_{0}}^{(L)}(x_{0}))\ddot{\sigma}(\alpha_{m_{1}}^{(L)}(x_{1}))\dot{\sigma}(\alpha_{m_{2}}^{(L)}(x_{2}))\sigma(\alpha_{m_{2}}^{(L)}(x_{3}))\\
W_{m_{0}k_{0}}^{(L)}W_{m_{1}k_{1}}^{(L)}\delta_{k_{2}k_{3}}
\end{aligned}
\\
 & +n_{L}^{-2}\sum_{m_{1},m_{2}}\Theta_{m_{1},m_{2}}^{(L)}(x_{1},x_{2})\sigma(\alpha_{m_{1}}^{(L)}(x_{0}))\dot{\sigma}(\alpha_{m_{1}}^{(L)}(x_{1}))\dot{\sigma}(\alpha_{m_{2}}^{(L)}(x_{2}))\sigma(\alpha_{m_{2}}^{(L)}(x_{3}))\delta_{k_{0}k_{1}}\delta_{k_{2}k_{3}}\\
 & \begin{aligned}+n_{L}^{-2}\sum_{m_{0},m_{1},m_{3}}\Theta_{m_{0},m_{1}}^{(L)}(x_{0},x_{1})\Theta_{m_{1},m_{3}}^{(L)}(x_{2},x_{3})\dot{\sigma}(\alpha_{m_{0}}^{(L)}(x_{0}))\dot{\sigma}(\alpha_{m_{1}}^{(L)}(x_{1}))\dot{\sigma}(\alpha_{m_{1}}^{(L)}(x_{2}))\dot{\sigma}(\alpha_{m_{3}}^{(L)}(x_{3}))\\
W_{m_{0}k_{0}}^{(L)}\delta_{k_{1}k_{2}}W_{m_{3}k_{3}}^{(L)}
\end{aligned}
\end{align*}
As $n_{1},...,n_{\ell-1}\to\infty$ and for any times $t<T$, the
NTK $\Theta^{(\ell)}$ converges to its limit while $\Omega^{(\ell)}$
and $\Gamma^{(\ell)}$ vanishes. $\Gamma_{k_{0},k_{1},k_{2},k_{3}}^{(L+1)}(x_{0},x_{1},x_{2},x_{3})$
therefore converges to:
\begin{align*}
 & \begin{aligned}+n_{L}^{-2}\sum_{m}\Theta_{\infty}^{(L)}(x_{0},x_{1})\Theta_{\infty}^{(L)}(x_{1},x_{2})\Theta_{\infty}^{(L)}(x_{2},x_{3})\dot{\sigma}(\alpha_{m}^{(L)}(x_{0}))\ddot{\sigma}(\alpha_{m}^{(L)}(x_{1}))\ddot{\sigma}(\alpha_{m}^{(L)}(x_{2}))\dot{\sigma}(\alpha_{m}^{(L)}(x_{3}))\\
W_{mk_{0}}^{(L)}W_{mk_{1}}^{(L)}W_{mk_{2}}^{(L)}W_{mk_{3}}^{(L)}
\end{aligned}
\\
 & \begin{aligned}+n_{L}^{-2}\sum_{m}\Theta_{\infty}^{(L)}(x_{1},x_{2})\Theta_{\infty}^{(L)}(x_{2},x_{3})\sigma(\alpha_{m}^{(L)}(x_{0}))\dot{\sigma}(\alpha_{m}^{(L)}(x_{1}))\ddot{\sigma}(\alpha_{m}^{(L)}(x_{2}))\dot{\sigma}(\alpha_{m}^{(L)}(x_{3}))\\
\delta_{k_{0}k_{1}}W_{mk_{2}}^{(L)}W_{mk_{3}}^{(L)}
\end{aligned}
\\
 & \begin{aligned}+n_{L}^{-2}\sum_{m}\Theta_{\infty}^{(L)}(x_{0},x_{1})\Theta_{\infty}^{(L)}(x_{1},x_{2})\dot{\sigma}(\alpha_{m}^{(L)}(x_{0}))\ddot{\sigma}(\alpha_{m}^{(L)}(x_{1}))\dot{\sigma}(\alpha_{m}^{(L)}(x_{2}))\sigma(\alpha_{m}^{(L)}(x_{3}))\\
W_{mk_{0}}^{(L)}W_{mk_{1}}^{(L)}\delta_{k_{2}k_{3}}
\end{aligned}
\\
 & \begin{aligned}+n_{L}^{-2}\sum_{m}\Theta_{\infty}^{(L)}(x_{1},x_{2})\sigma(\alpha_{m}^{(L)}(x_{0}))\dot{\sigma}(\alpha_{m}^{(L)}(x_{1}))\dot{\sigma}(\alpha_{m}^{(L)}(x_{2}))\sigma(\alpha_{m}^{(L)}(x_{3}))\delta_{k_{0}k_{1}}\delta_{k_{2}k_{3}}\end{aligned}
\\
 & \begin{aligned}+n_{L}^{-2}\sum_{m}\Theta_{\infty}^{(L)}(x_{0},x_{1})\Theta_{\infty}^{(L)}(x_{2},x_{3})\dot{\sigma}(\alpha_{m}^{(L)}(x_{0}))\dot{\sigma}(\alpha_{m}^{(L)}(x_{1}))\dot{\sigma}(\alpha_{m}^{(L)}(x_{2}))\dot{\sigma}(\alpha_{m}^{(L)}(x_{3}))\\
W_{mk_{0}}^{(L)}\delta_{k_{1}k_{2}}W_{mk_{3}}^{(L)}
\end{aligned}
\end{align*}

For the convergence during training, we proceed similarily to the
proof of Lemma \ref{lem:omega-vanishes}. At initialization, all terms
vanish as $n_{\ell}\to\infty$ because all summands are independent
(after taking the $n_{1},\ldots,n_{L-1}\to\infty$ limit) with zero
mean and finite variance. During training, the weights $W^{(\ell)}$
and preactivations $\tilde{\alpha}^{(\ell)}$ move at a rate of $\nicefrac{1}{\sqrt{n_{\ell}}}$
which leads to a change of order $n_{\ell}^{-2+\nicefrac{1}{2}}=n_{\ell}^{-1.5}$,
which vanishes for all times $t$ too.
\end{proof}

\section{The Matrix $S$\label{sec:The-Matrix-S}}

We now have the theoretical tools to describe the moments of the matrix
$S$. We first give a bound for the rank of $S$:
\begin{prop}
$Rank(S)\leq2(n_{1}+...+n_{L-1})Nn_{L}$
\end{prop}
\begin{proof}
We first observe that $S$ is given by a sum of $Nn_{L}$ matrices:
\[
S_{pp'}=\sum_{i=1}^{N}\sum_{k=1}^{n_{L}}\partial_{ik}C\partial_{\theta_{p}\theta_{p}}^{2}f_{\theta,k}(x_{i}).
\]
It is therefore sufficiant to show that the rank of each matrices
$\mathcal{H}f_{\theta,k}(x)=\left(\partial_{\theta_{p}\theta_{p'}}^{2}f_{\theta,k}(x_{i})\right)_{p,p'}$
is bounded by $2(n_{1}+...+n_{L})$.

The derivatives $\partial_{\theta_{p}}f_{\theta,k}(x)$ have different
definition depending on whether the parameter $\theta_{p}$ is a connection
weight $W_{ij}^{(\ell)}$ or a bias $b_{j}^{(\ell)}$:
\begin{align*}
\partial_{W_{ij}^{(\ell)}}f_{\theta,k}(x) & =\frac{1}{\sqrt{n_{\ell}}}\alpha_{i}^{(\ell)}(x;\theta)\partial_{\tilde{\alpha}_{j}^{(\ell+1)}(x;\theta)}f_{\theta,k}(x)\\
\partial_{b_{j}^{(\ell)}}f_{\theta,k}(x) & =\beta\partial_{\tilde{\alpha}_{j}^{(\ell+1)}(x;\theta)}f_{\theta,k}(x)
\end{align*}
These formulas only depend on $\theta$ through the values $\left(\alpha_{i}^{(\ell)}(x;\theta)\right)_{\ell,i}$
and $\left(\partial_{\tilde{\alpha}_{i}^{(\ell)}(x;\theta)}f_{\theta,k}(x)\right)_{\ell,i}$
for $\ell=1,...,L-1$ (note that both $\alpha_{i}^{(0)}(x)=x_{i}$
and $\partial_{\tilde{\alpha}_{i}^{(L)}(x;\theta)}f_{\theta,k}(x)=\delta_{ik}$
do not depend on $\theta$). Together there are $2(n_{1}+...+n_{L-1})$
of them. As a consequence, the map $\theta\mapsto\left(\partial_{\theta_{p}}f_{\theta,k}(x_{i})\right)_{p}$
can be written as a composition
\[
\theta\in\mathbb{R}^{P}\mapsto\left(\alpha_{i}^{(\ell)}(x;\theta),\partial_{\tilde{\alpha}_{i}^{(\ell)}(x;\theta)}f_{\theta,k}(x)\right)_{\ell,i}\in\mathbb{R}^{2(n_{1}+...+n_{L-1})}\mapsto\left(\partial_{\theta_{p}}f_{\theta,k}(x_{i})\right)_{p}\in\mathbb{R}^{P}
\]
and the matrix $\mathcal{H}f_{\theta,k}(x)$ is equal to the Jacobian
of this map. By the chain rule, $\mathcal{H}f_{\theta,k}(x)$ is the
matrix multiplication of the Jacobians of the two submaps, whose rank
are bounded by $2(n_{1}+...+n_{L-1})$, hence bounding the rank of
$\mathcal{H}f_{\theta,k}(x)$. And because $S$ is a sum of $Nn_{L}$
matrices of rank smaller than $2(n_{1}+...+n_{L-1})$, the rank of
$S$ is bounded by $2(n_{1}+...+n_{L-1})Nn_{L}$.
\end{proof}

\subsection{Moments}

Let us now prove Proposition \ref{prop:S_moments}:
\begin{prop}
\label{prop:S_moments}For any loss $C$ with BGOSS and $\sigma\in C_{b}^{4}(\mathbb{R})$,
the first two moments of $S$ take the form
\begin{align*}
\mathrm{Tr}\left(S(t)\right) & =G(t)^{T}\nabla C(t)\\
\mathrm{Tr}\left(S(t)^{2}\right) & =\nabla C(t)^{T}\tilde{\Upsilon}(t)\nabla C(t)
\end{align*}

- At initialization, $g_{\theta}$ and $f_{\theta}$ converge to a
(centered) Gaussian pair with covariances
\begin{align*}
\mathbb{E}[g_{\theta,k}(x)g_{\theta,k'}(x')] & =\delta_{kk'}\Xi_{\infty}^{(L)}(x,x')\\
\mathbb{E}[g_{\theta,k}(x)f_{\theta,k'}(x')] & =\delta_{kk'}\Phi_{\infty}^{(L)}(x,x')\\
\mathbb{E}[f_{\theta,k}(x)f_{\theta,k'}(x')] & =\delta_{kk'}\Sigma_{\infty}^{(L)}(x,x')
\end{align*}
 and during training $g_{\theta}$ evolves according to
\[
\partial_{t}g_{\theta,k}(x)=\sum_{i=1}^{N}\Lambda_{\infty}^{(L)}(x,x_{i})\partial_{ik}C(Y(t))_{\cdot}
\]

- Uniformly over any interval $[0,T]$ where $\int_{0}^{T}\left\Vert \nabla C(t)\right\Vert _{2}dt$
is stochastically bounded, the kernel $\Upsilon^{(L)}$ has a deterministic
and fixed limit $\lim_{n_{L-1}\to\infty}\cdots\lim_{n_{1}\to\infty}\Upsilon_{kk'}^{(L)}(x,x')=\delta_{kk'}\Upsilon_{\infty}^{(L)}(x,x')$
with limiting kernel:

\[
\Upsilon_{\infty}^{(L)}(x,x')=\sum_{\ell=1}^{L-1}\left(\Theta_{\infty}^{(\ell)}(x,x')^{2}\ddot{\Sigma}^{(\ell)}(x,x')+2\Theta_{\infty}^{(\ell)}(x,x')\dot{\Sigma}^{(\ell)}(x,x')\right)\dot{\Sigma}^{(\ell+1)}(x,x')\cdots\dot{\Sigma}^{(L-1)}(x,x').
\]

- The higher moment $k>2$ vanish: $\lim_{n_{L-1}\to\infty}\cdots\lim_{n_{1}\to\infty}\mathrm{Tr}\left(S^{k}\right)=0$.
\end{prop}
\begin{proof}
The first moment of $S$ takes the form
\[
\mathrm{Tr}\left(S\right)=\sum_{p}\left(\nabla C\right)^{T}\mathcal{H}_{p,p}Y=\left(\nabla C\right)^{T}G
\]
 where $G$ is the restriction to the training set of the function
$g_{\theta}(x)=\sum_{p}\partial_{\theta_{p}\theta_{p}}^{2}f_{\theta}(x)$.
This process is random at initialization and varies during training.
Lemma \ref{prop:moment1_S} below shows that, in the infinite width
limit, it is a Gaussian process at initialization which then evolves
according to a simple differential equation, hence describing the
evolution of the first moment during training.

The second moment of $S$ takes the form:
\begin{align*}
\mathrm{Tr}(S^{2}) & =\sum_{p_{1},p_{2}=1}^{P}\sum_{i_{1},i_{2}=1}^{N}\partial_{\theta_{p_{1}},\theta_{p_{2}}}^{2}f_{\theta,k_{1}}(x_{1})\partial_{\theta_{p_{2}},\theta_{p_{1}}}^{2}f_{\theta,k_{2}}(x_{2})c_{i_{1}}'(x_{i_{1}})c_{i_{2}}'(x_{i_{2}})\\
 & =\left(\nabla C\right)^{T}\tilde{\Upsilon}\nabla C
\end{align*}
where $\Upsilon_{k_{1},k_{2}}^{(L)}(x_{1},x_{2})=\sum_{p_{1},p_{2}=1}^{P}\partial_{\theta_{p_{1}},\theta_{p_{2}}}^{2}f_{\theta,k_{1}}(x_{1})\partial_{\theta_{p_{2}},\theta_{p_{1}}}^{2}f_{\theta,k_{2}}(x_{2})$
is a multidimensional kernel and $\tilde{\Upsilon}$ is its Gram matrix.
Lemma \ref{prop:moment2_S} below shows that in the infinite-width
limit, $\Upsilon_{k_{1},k_{2}}^{(L)}(x_{1},x_{2})$ converges to a
deterministic and time-independent limit $\Upsilon_{\infty}^{(L)}(x_{1},x_{2})\delta_{k_{1}k_{2}}$.

To show that $\mathrm{Tr}(S^{k})\to0$ for all $k>2$, it suffices
to show that $\left\Vert S^{2}\right\Vert _{F}\to0$ as $\left|\mathrm{Tr}(S^{k})\right|<\left\Vert S^{2}\right\Vert _{F}\left\Vert S\right\Vert _{F}^{k-2}$
and we know that $\left\Vert S\right\Vert _{F}\to\left(\partial_{Y}C\right)^{T}\tilde{\Upsilon}\partial_{Y}C$
is finite. We have that 
\begin{flalign*}
\left\Vert S^{2}\right\Vert _{F} & =\sum_{i_{0},i_{1},i_{2},i_{3}=1}^{N}\sum_{k_{0},k_{1},k_{2},k_{3}=1}^{n_{L}}\Psi_{k_{0},k_{1},k_{2},k_{3}}^{(L)}(x_{i_{0}},x_{i_{1}},x_{i_{2}},x_{i_{3}})\partial_{f_{\theta,k_{0}}(x_{i_{0}})}C\partial_{f_{\theta,k_{1}}(x_{i_{1}})}C\\
 & \hspace{8.15cm}\partial_{f_{\theta,k_{2}}(x_{i_{2}})}C\partial_{f_{\theta,k_{3}}(x_{i_{3}})}C\\
 & =\tilde{\Psi}\cdot\left(\partial_{Y}C\right)^{\otimes4}
\end{flalign*}
for $\tilde{\Psi}$ the $Nn_{L}\times Nn_{L}\times Nn_{L}\times Nn_{L}$
finite version of 
\begin{align*}
\Psi_{k_{0},k_{1},k_{2},k_{3}}^{(L)}(x_{i_{0}},x_{i_{1}},x_{i_{2}},x_{i_{3}}) & =\sum_{p_{0},p_{1},p_{2},p_{3}=1}^{P}\partial_{\theta_{p_{0}},\theta_{p_{1}}}^{2}f_{\theta,k_{0}}(x_{0})\partial_{\theta_{p_{1}},\theta_{p_{2}}}^{2}f_{\theta,k_{1}}(x_{1})\\
 & \hspace{2.35cm}\partial_{\theta_{p_{2}},\theta_{p_{3}}}^{2}f_{\theta,k_{2}}(x_{2})\partial_{\theta_{p_{3}},\theta_{p_{0}}}^{2}f_{\theta,k_{3}}(x_{3}).
\end{align*}
which vanishes in the infinite width limit by Lemma \ref{lem:momentk_S}
below.
\end{proof}
\begin{lem}
\label{prop:moment1_S}For any loss $C$ with BGOSS and $\sigma\in C_{b}^{4}(\mathbb{R})$,
at initialization $g_{\theta}$ and $f_{\theta}$ converge to a (centered)
Gaussian pair with covariances
\begin{align*}
\mathbb{E}[g_{\theta,k}(x)g_{\theta,k'}(x')] & =\delta_{kk'}\Xi_{\infty}^{(L)}(x,x')\\
\mathbb{E}[g_{\theta,k}(x)f_{\theta,k'}(x')] & =\delta_{kk'}\Phi_{\infty}^{(L)}(x,x')\\
\mathbb{E}[f_{\theta,k}(x)f_{\theta,k'}(x')] & =\delta_{kk'}\Sigma_{\infty}^{(L)}(x,x')
\end{align*}
 and during training $g_{\theta}$ evolves according to
\[
\partial_{t}g_{\theta}(x)=\sum_{i=1}^{N}\Lambda_{\infty}^{(L)}(x,x_{i})D_{i}(t)
\]
\end{lem}
\begin{proof}
When $L=1$, $g_{\theta}(x)$ is $0$ for any $x$ and $\theta$.

For the inductive step, the trace $g_{\theta,k}^{(L+1)}(x)$ is defined
recursively as
\[
\frac{1}{\sqrt{n_{L}}}\sum_{m=1}^{n_{L}}g_{\theta,m}^{(L)}(x)\dot{\sigma}\left(\tilde{\alpha}_{m}^{(L)}(x)\right)W_{mk}^{(L)}+\mathrm{Tr}\left(\nabla f_{\theta,m}(x)\left(\nabla f_{\theta,m}(x)\right)^{T}\right)\ddot{\sigma}\left(\tilde{\alpha}_{m}^{(L)}(x)\right)W_{mk}^{(L)}
\]
First note that $\mathrm{Tr}\left(\nabla f_{\theta,m}(x)\left(\nabla f_{\theta,m}(x)\right)^{T}\right)=\Theta_{mm}^{(L)}(x,x)$.
Now let $n_{1},...n_{L-1}\to\infty$, by the induction hypothesis,
the pairs $(g_{\theta,m}^{(L)},\tilde{\alpha}_{m}^{(L)})$ converge
to iid Gaussian pairs of processes with covariance $\Phi_{\infty}^{(L)}$
at initialization.

At initialization, conditioned on the values of $g_{m}^{(L)},\tilde{\alpha}_{m}^{(L)}$
the pairs $(g_{k}^{(L+1)},f_{\theta})$ follow a centered Gaussian
distribution with (conditioned) covariance
\begin{align*}
\mathbb{E}[g_{\theta,k}^{(L+1)}(x)g_{\theta,k'}^{(L+1)}(x')|g_{\theta,m}^{(L)},\tilde{\alpha}_{m}^{(L)}] & =\frac{\delta_{kk'}}{n_{L}}\sum_{m=1}^{n_{L}}\left(g_{\theta,m}^{(L)}(x)\dot{\sigma}\left(\tilde{\alpha}_{m}^{(L)}(x)\right)+\Theta_{\infty}^{(L)}(x,x)\ddot{\sigma}\left(\tilde{\alpha}_{m}^{(L)}(x)\right)\right)\\
 & \hspace{45pt}\left(g_{\theta,m}^{(L)}(x')\dot{\sigma}\left(\tilde{\alpha}_{m}^{(L)}(x')\right)+\Theta_{\infty}^{(L)}(x',x')\ddot{\sigma}\left(\tilde{\alpha}_{m}^{(L)}(x')\right)\right)\\
\mathbb{E}[g_{\theta,k}^{(L+1)}(x)f_{\theta,k'}(x')|g_{\theta,m}^{(L)},\tilde{\alpha}_{m}^{(L)}] & =\frac{\delta_{kk'}}{n_{L}}\sum_{m=1}^{n_{L}}\left(g_{\theta,m}^{(L)}(x)\dot{\sigma}\left(\tilde{\alpha}_{m}^{(L)}(x)\right)+\Theta_{\infty}^{(L)}(x,x)\ddot{\sigma}\left(\tilde{\alpha}_{m}^{(L)}(x)\right)\right)\\
 & \hspace{200pt}\sigma\left(\tilde{\alpha}_{m}^{(L)}(x')\right)\\
\mathbb{E}[f_{\theta,k}(x)f_{\theta,k'}(x')|g_{\theta,m}^{(L)},\tilde{\alpha}_{m}^{(L)}] & =\frac{\delta_{kk'}}{n_{L}}\sum_{m=1}^{n_{L}}\sigma\left(\tilde{\alpha}_{m}^{(L)}(x)\right)\sigma\left(\tilde{\alpha}_{m}^{(L)}(x')\right)+\beta^{2}.
\end{align*}
As $n_{L}\to\infty$, by the law of large number, these (random) covariances
converge to their expectations which are deterministic, hence the
pairs $(g_{k}^{(L+1)},f_{\theta k})$ have asymptotically the same
Gaussian distribution independent of $g_{m}^{(L)},\tilde{\alpha}_{m}^{(L)}$:
\begin{align*}
\mathbb{E}\left[g_{\theta,k}^{(L)}(x)g_{\theta,k'}^{(L)}(x')\right] & \to\delta_{kk'}\Xi_{\infty}^{(L)}(x,x')\\
\mathbb{E}\left[g_{\theta,k}^{(L)}(x)f_{\theta,k'}^{(L)}(x')\right] & \to\delta_{kk'}\Phi_{\infty}^{(L)}(x,x)\\
\mathbb{E}\left[f_{\theta,k}^{(L)}(x)f_{\theta,k'}^{(L)}(x')\right] & \to\delta_{kk'}\Sigma_{\infty}^{(L)}(x,x)
\end{align*}
 with $\Xi_{\infty}^{(1)}(x,x')=\Phi_{\infty}^{(1)}(x,x')=0$ and
\begin{align*}
\Xi_{\infty}^{(L+1)}(x,x') & =\mathbb{E}\left[gg'\dot{\sigma}(\alpha)\dot{\sigma}(\alpha')\right]\\
 & +\Theta_{\infty}^{(L)}(x',x')\mathbb{E}\left[g\dot{\sigma}(\alpha)\ddot{\sigma}(\alpha')\right]\\
 & +\Theta_{\infty}^{(L)}(x,x)\mathbb{E}\left[g'\dot{\sigma}(\alpha')\ddot{\sigma}(\alpha)\right]\\
 & +\Theta_{\infty}^{(L)}(x,x)\Theta_{\infty}^{(L)}(x',x')\mathbb{E}\left[\ddot{\sigma}(\alpha')\ddot{\sigma}(\alpha)\right]\\
 & =\Xi_{\infty}^{(L)}(x,x')\dot{\Sigma}_{\infty}^{(L)}(x,x')+\left(\Phi_{\infty}^{(L)}(x,x')\Phi_{\infty}^{(L)}(x',x)+\Phi_{\infty}^{(L)}(x,x)\Phi_{\infty}^{(L)}(x',x')\right)\ddot{\Sigma}_{\infty}^{(L)}(x,x')\\
 & +\Phi_{\infty}^{(L)}(x,x')\Phi_{\infty}^{(L)}(x',x')\mathbb{E}\left[\dot{\sigma}(\alpha)\dddot{\sigma}(\alpha')\right]+\Phi_{\infty}^{(L)}(x,x)\Phi_{\infty}^{(L)}(x',x)\mathbb{E}\left[\dddot{\sigma}(\alpha)\dot{\sigma}(\alpha')\right]\\
 & +\Theta_{\infty}^{(L)}(x',x')\left(\Phi_{\infty}^{(L)}(x,x)\ddot{\Sigma}_{\infty}^{(L)}(x,x')+\Phi_{\infty}^{(L)}(x,x')\mathbb{E}\left[\dot{\sigma}(\alpha)\dddot{\sigma}(\alpha')\right]\right)\\
 & +\Theta_{\infty}^{(L)}(x,x)\left(\Phi_{\infty}^{(L)}(x',x')\ddot{\Sigma}_{\infty}^{(L)}(x,x')+\Phi_{\infty}^{(L)}(x',x)\mathbb{E}\left[\dddot{\sigma}(\alpha)\dot{\sigma}(\alpha')\right]\right)\\
 & +\Theta_{\infty}^{(L)}(x,x)\Theta_{\infty}^{(L)}(x',x')\ddot{\Sigma}_{\infty}^{(L)}(x,x')
\end{align*}
and 
\begin{align*}
\Phi_{\infty}^{(L+1)}(x,x') & =\mathbb{E}\left[g\dot{\sigma}(\alpha)\sigma(\alpha')\right]+\Theta_{\infty}^{(L)}(x,x)\mathbb{E}\left[\ddot{\sigma}(\alpha)\sigma(\alpha')\right]\\
 & =\Phi_{\infty}^{(L)}(x,x')\dot{\Sigma}^{(L+1)}(x,x')+\left(\Phi_{\infty}^{(L)}(x,x)+\Theta_{\infty}^{(L)}(x,x)\right)\mathbb{E}\left[\ddot{\sigma}(\alpha)\sigma(\alpha')\right]
\end{align*}
 where $(g,g',\alpha,\alpha')$ is a Gaussian quadruple of covariance
\[
\left(\begin{array}{cccc}
\Xi_{\infty}^{(L)}(x,x) & \Xi_{\infty}^{(L)}(x,x') & \Phi_{\infty}^{(L)}(x,x) & \Phi_{\infty}^{(L)}(x,x')\\
\Xi_{\infty}^{(L)}(x,x') & \Xi_{\infty}^{(L)}(x',x') & \Phi_{\infty}^{(L)}(x',x) & \Phi_{\infty}^{(L)}(x',x')\\
\Phi_{\infty}^{(L)}(x,x) & \Phi_{\infty}^{(L)}(x',x) & \Sigma_{\infty}^{(L)}(x,x) & \Sigma_{\infty}^{(L)}(x,x')\\
\Phi_{\infty}^{(L)}(x,x') & \Phi_{\infty}^{(L)}(x',x') & \Sigma_{\infty}^{(L)}(x,x') & \Sigma_{\infty}^{(L)}(x',x')
\end{array}\right).
\]

During training, the parameters follow the gradient $\partial_{t}\theta(t)=\left(\partial_{\theta}Y(t)\right)^{T}D(t)$.
By the induction hypothesis, the traces $g_{\theta,m}^{(L)}$ then
evolve according to the differential equation
\[
\partial_{t}g_{\theta,m}^{(L)}(x)=\frac{1}{\sqrt{n_{L}}}\sum_{i=1}^{N}\sum_{m=1}^{n_{L}}\Lambda_{mm'}^{(L)}(x,x_{i})\dot{\sigma}(\tilde{\alpha}_{m'}^{(L)}(x))\left(W_{m'}^{(L)}\right)^{T}D_{i}(t)
\]
and in the limit as $n_{1},...,n_{L-1}\to\infty$, the kernel $\Lambda_{mm'}^{(L)}(x,x_{i})$
converges to a deterministic and fixed limit $\delta_{mm'}\Lambda_{\infty}^{(L)}(x,x_{i})$.
Note that as $n_{L}$ grows, the $g_{\theta,m}^{(L)}(x)$ move at
a rate of $\nicefrac{1}{\sqrt{n_{L}}}$ just like the pre-activations
$\tilde{\alpha}_{m}^{(L)}$. Even though they move less and less,
together they affect the trace $g_{\theta,k}^{(L+1)}$ which follows
the differential equation
\[
\partial_{t}g_{\theta,k}^{(L+1)}(x)=\sum_{i=1}^{N}\sum_{k'=1}^{n_{L}}\Lambda_{kk'}^{(L+1)}(x,x_{i})D_{ik'}(t)
\]
where
\begin{align*}
\Lambda_{kk'}^{(L+1)}(x,x') & =\frac{1}{n_{L}}\sum_{m,m'}\Lambda_{mm'}^{(L)}(x,x')\dot{\sigma}\left(\tilde{\alpha}_{m}^{(L)}(x)\right)\dot{\sigma}\left(\tilde{\alpha}_{m'}^{(L)}(x')\right)W_{mk}^{(L)}W_{m'k'}^{(L)}\\
 & +\frac{1}{n_{L}}\sum_{m,m'}g_{\theta,m}^{(L)}(x)\Theta_{mm'}^{(L)}(x,x')\ddot{\sigma}\left(\tilde{\alpha}_{m}^{(L)}(x)\right)\dot{\sigma}\left(\tilde{\alpha}_{m'}^{(L)}(x')\right)W_{mk}^{(L)}W_{m'k'}^{(L)}\\
 & +\frac{1}{n_{L}}\sum_{m}g_{\theta,m}^{(L)}(x)\dot{\sigma}\left(\tilde{\alpha}_{m}^{(L)}(x)\right)\sigma\left(\tilde{\alpha}_{m}^{(L)}(x')\right)\delta_{kk'}\\
 & +\frac{2}{n_{L}}\sum_{m,m'}\Omega_{m'mm}^{(L)}(x',x,x)\ddot{\sigma}\left(\tilde{\alpha}_{m}^{(L)}(x)\right)\dot{\sigma}\left(\tilde{\alpha}_{m'}^{(L)}(x')\right)W_{mk}^{(L)}W_{m'k'}^{(L)}\\
 & +\frac{1}{n_{L}}\sum_{m,m'}\Theta_{mm}^{(L)}(x,x)\Theta_{mm'}^{(L)}(x,x')\dddot{\sigma}\left(\tilde{\alpha}_{m}^{(L)}(x)\right)\dot{\sigma}\left(\tilde{\alpha}_{m'}^{(L)}(x')\right)W_{mk}^{(L)}W_{m'k'}^{(L)}\\
 & +\frac{1}{n_{L}}\sum_{m}\Theta_{mm}^{(L)}(x,x)\ddot{\sigma}\left(\tilde{\alpha}_{m}^{(L)}(x)\right)\sigma\left(\tilde{\alpha}_{m}^{(L)}(x')\right)\delta_{kk'}.
\end{align*}
As $n_{1},...,n_{L-1}\to\infty$, the kernels $\Theta_{mm'}^{(L)}(x,x')$
and $\Lambda_{mm'}^{(L)}(x,x')$ converge to their limit and $\Omega_{m'mm}^{(L)}(x',x,x)$
vanishes:
\begin{align*}
\Lambda_{kk'}^{(L)}(x,x') & \to\frac{1}{n_{L}}\sum_{m}\Lambda_{\infty}^{(L)}(x,x')\dot{\sigma}\left(\tilde{\alpha}_{m}^{(L)}(x)\right)\dot{\sigma}\left(\tilde{\alpha}_{m}^{(L)}(x')\right)W_{mk}^{(L)}W_{mk'}^{(L)}\\
 & +\frac{1}{n_{L}}\sum_{m}g_{\theta,m}^{(L)}(x)\Theta_{\infty}^{(L)}(x,x')\ddot{\sigma}\left(\tilde{\alpha}_{m}^{(L)}(x)\right)\dot{\sigma}\left(\tilde{\alpha}_{m}^{(L)}(x')\right)W_{mk}^{(L)}W_{mk'}^{(L)}\\
 & +\frac{1}{n_{L}}\sum_{m}g_{\theta,m}^{(L)}(x)\dot{\sigma}\left(\tilde{\alpha}_{m}^{(L)}(x)\right)\sigma\left(\tilde{\alpha}_{m}^{(L)}(x')\right)\delta_{kk'}\\
 & +\frac{1}{n_{L}}\sum_{m}\Theta_{\infty}^{(L)}(x,x)\Theta_{\infty}^{(L)}(x,x')\dddot{\sigma}\left(\tilde{\alpha}_{m}^{(L)}(x)\right)\dot{\sigma}\left(\tilde{\alpha}_{m}^{(L)}(x')\right)W_{mk}^{(L)}W_{mk'}^{(L)}\\
 & +\frac{1}{n_{L}}\sum_{m}\Theta_{\infty}^{(L)}(x,x)\ddot{\sigma}\left(\tilde{\alpha}_{m}^{(L)}(x)\right)\sigma\left(\tilde{\alpha}_{m}^{(L)}(x')\right)\delta_{kk'}
\end{align*}
By the law of large numbers, as $n_{L}\to\infty$, at initialization
$\Lambda_{kk'}^{(L+1)}(x,x')\to\delta_{kk'}\Lambda_{\infty}^{(L+1)}(x,x')$
where 
\begin{align*}
\Lambda_{\infty}^{(L+1)}(x,x') & =\Lambda_{\infty}^{(L)}(x,x')\dot{\Sigma}_{\infty}^{(L+1)}(x,x')\\
 & +\Theta_{\infty}^{(L)}(x,x')\mathbb{E}\left[g\ddot{\sigma}\left(\alpha\right)\dot{\sigma}\left(\alpha'\right)\right]\\
 & +\mathbb{E}\left[g\dot{\sigma}\left(\alpha\right)\sigma\left(\alpha'\right)\right]\\
 & +\Theta_{\infty}^{(L)}(x,x)\Theta_{\infty}^{(L)}(x,x')\mathbb{E}\left[\dddot{\sigma}\left(\alpha\right)\dot{\sigma}\left(\alpha'\right)\right]\\
 & +\Theta_{\infty}^{(L)}(x,x)\mathbb{E}\left[\ddot{\sigma}\left(\alpha\right)\sigma\left(\alpha'\right)\right]\\
 & =\Lambda_{\infty}^{(L)}(x,x')\dot{\Sigma}_{\infty}^{(L+1)}(x,x')\\
 & +\Theta_{\infty}^{(L)}(x,x')\left(\Phi_{\infty}^{(L)}(x,x')\ddot{\Sigma}_{\infty}^{(L+1)}(x,x')+\Phi_{\infty}^{(L)}(x,x)\mathbb{E}\left[\dddot{\sigma}\left(\alpha\right)\dot{\sigma}\left(\alpha'\right)\right]\right)\\
 & +\Phi_{\infty}^{(L)}(x,x')\dot{\Sigma}_{\infty}^{(L+1)}(x,x')+\Phi_{\infty}^{(L)}(x,x)\mathbb{E}\left[\ddot{\sigma}\left(\alpha\right)\sigma\left(\alpha'\right)\right]\\
 & +\Theta_{\infty}^{(L)}(x,x)\Theta_{\infty}^{(L)}(x,x')\mathbb{E}\left[\dddot{\sigma}\left(\alpha\right)\dot{\sigma}\left(\alpha'\right)\right]\\
 & +\Theta_{\infty}^{(L)}(x,x)\mathbb{E}\left[\ddot{\sigma}\left(\alpha\right)\dot{\sigma}\left(\alpha'\right)\right]
\end{align*}

During training $\Theta_{\infty}^{(L)}$ and $\Lambda_{\infty}^{(L)}$
are fixed in the limit $n_{1},..,n_{L-1}\to\infty$, and the values
$g_{\theta,m}^{(L)}(x)$, $\tilde{\alpha}_{m}^{(L)}(x)$ and $W_{mk}^{(L)}$
vary at a rate of $\nicefrac{1}{\sqrt{n_{L}}}$ which induce a change
of the same rate to $\Lambda_{kk'}^{(L)}(x,x')$, which is therefore
asymptotically fixed during training as $n_{L}\to\infty$.
\end{proof}
The next lemma describes the asymptotic limit of the kernel $\Upsilon^{(L)}$:
\begin{lem}
\label{prop:moment2_S}For any loss $C$ with BGOSS and $\sigma\in C_{b}^{4}(\mathbb{R})$,
the second moment of the Hessian of the realization function $\mathcal{H}F^{(L)}$
converges uniformly over $[0,T]$ to a fixed limit as $n_{1},...n_{L-1}\to\infty$
\[
\Upsilon_{kk'}^{(L)}(x,x')\to\delta_{kk'}\sum_{\ell=1}^{L-1}\left(\Theta_{\infty}^{(\ell)}(x,x')^{2}\ddot{\Sigma}_{\infty}^{(\ell)}(x,x')+2\Theta_{\infty}^{(\ell)}(x,x')\dot{\Sigma}_{\infty}^{(\ell)}(x,x')\right)\dot{\Sigma}_{\infty}^{(\ell+1)}(x,x')\cdots\dot{\Sigma}_{\infty}^{(L-1)}(x,x').
\]
\end{lem}
\begin{proof}
The proof is by induction on the depth $L$. The case $L=1$ is trivially
true because $\partial_{\theta_{p}\theta_{p'}}^{2}f_{\theta,k}(x)=0$
for all $p,p',k,x$. For the induction step we observe that 
\begin{align*}
 & \Upsilon_{k,k'}^{(L)}(x,x')\\
 & =\sum_{p_{1},p_{2}=1}^{P}\partial_{\theta_{p_{1}},\theta_{p_{2}}}^{2}f_{\theta,k}(x)\partial_{\theta_{p_{2}},\theta_{p_{1}}}^{2}f_{\theta,k'}(x')\\
 & =\frac{1}{n_{L}}\sum_{m,m'=1}^{n_{L}}\Upsilon_{m,m'}^{(L)}(x,x')\dot{\sigma}\left(\tilde{\alpha}_{m}^{(L)}(x)\right)\dot{\sigma}\left(\tilde{\alpha}_{m'}^{(L)}(x')\right)W_{mk}^{(L)}W_{m'k'}^{(L)}\\
 & +\frac{1}{n_{L}}\sum_{m,m'=1}^{n_{L}}\Omega_{m',m,m'}^{(L)}(x',x,x')\dot{\sigma}\left(\tilde{\alpha}_{m}^{(L)}(x)\right)\ddot{\sigma}\left(\tilde{\alpha}_{m'}^{(L)}(x')\right)W_{mk}^{(L)}W_{m'k'}^{(L)}\\
 & +\frac{1}{n_{L}}\sum_{m,m'=1}^{n_{L}}\Omega_{m,m',m}^{(L)}(x,x',x)\ddot{\sigma}\left(\tilde{\alpha}_{m}^{(L)}(x)\right)\dot{\sigma}\left(\tilde{\alpha}_{m'}^{(L)}(x')\right)W_{mk}^{(L)}W_{m'k'}^{(L)}\\
 & +\frac{1}{n_{L}}\sum_{m,m'=1}^{n_{L}}\Theta_{m,m'}^{(L)}(x,x')\Theta_{m',m}^{(L)}(x',x)\ddot{\sigma}\left(\tilde{\alpha}_{m}^{(L)}(x)\right)\ddot{\sigma}\left(\tilde{\alpha}_{m'}^{(L)}(x')\right)W_{mk}^{(L)}W_{m'k'}^{(L)}\\
 & +\frac{2}{n_{L}}\sum_{m=1}^{n_{L}}\Theta_{m,m'}^{(L)}(x,x')\dot{\sigma}\left(\tilde{\alpha}_{m}^{(L)}(x)\right)\dot{\sigma}\left(\tilde{\alpha}_{m'}^{(L)}(x')\right)\delta_{kk'}
\end{align*}
if we now let the width of the lower layers grow to infinity $n_{1},...n_{L-1}\to\infty$,
the tensor $\Omega^{(L)}$ vanishes and $\Upsilon_{m,m'}^{(L)}$ and
the NTK $\Theta_{m,m'}^{(L)}$ converge to limits which are non-zero
only when $m=m'$. As a result, the term above converges to 
\begin{align*}
 & \frac{1}{n_{L}}\sum_{m=1}^{n_{L}}\Upsilon_{\infty}^{(L)}(x,x')\dot{\sigma}\left(\tilde{\alpha}_{m}^{(L)}(x)\right)\dot{\sigma}\left(\tilde{\alpha}_{m}^{(L)}(x')\right)W_{mk}^{(L)}W_{mk'}^{(L)}\\
 & +\frac{1}{n_{L}}\sum_{m=1}^{n_{L}}\Theta_{\infty}^{(L)}(x,x')^{2}\ddot{\sigma}\left(\tilde{\alpha}_{m}^{(L)}(x)\right)\ddot{\sigma}\left(\tilde{\alpha}_{m}^{(L)}(x')\right)W_{mk}^{(L)}W_{mk'}^{(L)}\\
 & +\frac{2}{n_{L}}\sum_{m=1}^{n_{L}}\Theta_{\infty}^{(L)}(x,x')\dot{\sigma}\left(\tilde{\alpha}_{m}^{(L)}(x)\right)\dot{\sigma}\left(\tilde{\alpha}_{m}^{(L)}(x')\right)\delta_{kk'}
\end{align*}
At initialization, we can apply the law of large numbers as $n_{L}\to\infty$
such that it converges to $\Upsilon_{\infty}^{(L+1)}(x,x')\delta_{kk'}$,
for the kernel $\Upsilon_{\infty}^{(L+1)}(x,x')$ defined recursively
by 
\begin{align*}
\Upsilon_{\infty}^{(L+1)}(x,x')= & \Upsilon_{\infty}^{(L)}(x,x')\dot{\Sigma}_{\infty}^{(L)}(x,x')+\Theta_{\infty}^{(L)}(x,x')^{2}\ddot{\Sigma}_{\infty}^{(L)}(x,x')+2\Theta_{\infty}^{(L)}(x,x')\dot{\Sigma}_{\infty}^{(L)}(x,x')
\end{align*}
and $\Upsilon_{\infty}^{(1)}(x,x')=0$. 

For the convergence during training, we proceed similarily to the
proof of Lemma \ref{lem:omega-vanishes}: the activations $\tilde{\alpha}_{m}^{(L)}(x)$
and weights $W_{mk}^{(L)}$ move at a rate of $\nicefrac{1}{\sqrt{n_{L}}}$
and the change to $\Upsilon_{kk'}^{(L+1)}$ is therefore of order
$\nicefrac{1}{\sqrt{n_{L}}}$ and vanishes as $n_{L}\to0$.
\end{proof}

Finally, the next lemma shows the vanishing of the tensor $\Psi_{k_{0},k_{1},k_{2},k_{3}}^{(L)}$
to prove that the higher moments of $S$ vanish.
\begin{lem}
\label{lem:momentk_S}For any loss $C$ with BGOSS and $\sigma\in C_{b}^{4}(\mathbb{R})$,
uniformly over $[0,T]$ 
\[
\lim_{n_{L-1}\to\infty}\cdots\lim_{n_{1}\to\infty}\Psi_{k_{0},k_{1},k_{2},k_{3}}^{(L)}(x_{i_{0}},x_{i_{1}},x_{i_{2}},x_{i_{3}})=0
\]
\end{lem}
\begin{proof}
When $L=1$ the Hessian is zero and $\Psi_{k_{0},k_{1},k_{2},k_{3}}^{(1)}(x_{i_{0}},x_{i_{1}},x_{i_{2}},x_{i_{3}})=0$.

For the induction step, we write $\Psi_{k_{0},k_{1},k_{2},k_{3}}^{(L+1)}(x_{i_{0}},x_{i_{1}},x_{i_{2}},x_{i_{3}})$
recursively, because it contains many terms, we change the notation,
writing $\left[\begin{array}{cc}
x_{0} & x_{1}\\
m_{0} & m_{1}
\end{array}\right]$ for $\Theta_{m_{0},m_{1}}^{(L)}(x_{0},x_{1})$, $\left[\begin{array}{ccc}
x_{0} & x_{1} & x_{2}\\
m_{0} & m_{1} & m_{2}
\end{array}\right]$ for $\Omega_{m_{0},m_{1},m_{2}}^{(L)}(x_{0},x_{1},x_{2})$ and $\left[\begin{array}{cccc}
x_{0} & x_{1} & x_{2} & x_{3}\\
m_{0} & m_{1} & m_{2} & m_{3}
\end{array}\right]$ for $\Gamma_{m_{0},m_{1},m_{2},m_{3}}^{(L)}(x_{0},x_{1},x_{2},x_{3})$.
The value $\Psi_{k_{0},k_{1},k_{2},k_{3}}^{(L+1)}(x_{i_{0}},x_{i_{1}},x_{i_{2}},x_{i_{3}})$
is then equal to %
\begin{align*}
 & \begin{aligned}n_{L}^{-2}\sum_{m_{0},m_{1},m_{2},m_{3}}\Psi_{m_{0},m_{1},m_{2},m_{3}}^{(L)}(x_{0},x_{1},x_{2},x_{3})\dot{\sigma}\left(\tilde{\alpha}_{m_{0}}^{(L)}(x_{0})\right)\dot{\sigma}\left(\tilde{\alpha}_{m_{1}}^{(L)}(x_{1})\right)\dot{\sigma}\left(\tilde{\alpha}_{m_{2}}^{(L)}(x_{2})\right)\\
\dot{\sigma}\left(\tilde{\alpha}_{m_{3}}^{(L)}(x_{3})\right)W_{m_{0}k_{0}}^{(L)}W_{m_{1}k_{1}}^{(L)}W_{m_{2}k_{2}}^{(L)}W_{m_{3}k_{3}}^{(L)}
\end{aligned}
\\
 & \begin{aligned}+n_{L}^{-2}\sum_{m_{0},m_{1},m_{2},m_{3}}\left[\begin{array}{cc}
x_{0} & x_{1}\\
m_{0} & m_{1}
\end{array}\right]\left[\begin{array}{cc}
x_{1} & x_{2}\\
m_{1} & m_{2}
\end{array}\right]\left[\begin{array}{cc}
x_{2} & x_{3}\\
m_{2} & m_{3}
\end{array}\right]\left[\begin{array}{cc}
x_{3} & x_{0}\\
m_{3} & m_{0}
\end{array}\right]\ddot{\sigma}\left(\tilde{\alpha}_{m_{0}}^{(L)}(x_{0})\right)\\
\ddot{\sigma}\left(\tilde{\alpha}_{m_{1}}^{(L)}(x_{1})\right)\ddot{\sigma}\left(\tilde{\alpha}_{m_{2}}^{(L)}(x_{2})\right)\ddot{\sigma}\left(\tilde{\alpha}_{m_{3}}^{(L)}(x_{3})\right)W_{m_{0}k_{0}}^{(L)}W_{m_{1}k_{1}}^{(L)}W_{m_{2}k_{2}}^{(L)}W_{m_{3}k_{3}}^{(L)}
\end{aligned}
\\
 & \begin{aligned}+n_{L}^{-2}\sum_{m_{0},m_{1},m_{2},m_{3}}\left[\begin{array}{ccc}
x_{0} & x_{1} & x_{2}\\
m_{0} & m_{1} & m_{2}
\end{array}\right]\left[\begin{array}{cc}
x_{2} & x_{3}\\
m_{2} & m_{3}
\end{array}\right]\left[\begin{array}{cc}
x_{3} & x_{0}\\
m_{3} & m_{0}
\end{array}\right]\ddot{\sigma}\left(\tilde{\alpha}_{m_{0}}^{(L)}(x_{0})\right)\dot{\sigma}\left(\tilde{\alpha}_{m_{1}}^{(L)}(x_{1})\right)\\
\ddot{\sigma}\left(\tilde{\alpha}_{m_{2}}^{(L)}(x_{2})\right)\ddot{\sigma}\left(\tilde{\alpha}_{m_{3}}^{(L)}(x_{3})\right)W_{m_{0}k_{0}}^{(L)}W_{m_{1}k_{1}}^{(L)}W_{m_{2}k_{2}}^{(L)}W_{m_{3}k_{3}}^{(L)}
\end{aligned}
\\
 & \begin{aligned}+n_{L}^{-2}\sum_{m_{0},m_{1},m_{2},m_{3}}\left[\begin{array}{cc}
x_{0} & x_{1}\\
m_{0} & m_{1}
\end{array}\right]\left[\begin{array}{ccc}
x_{1} & x_{2} & x_{3}\\
m_{1} & m_{2} & m_{3}
\end{array}\right]\left[\begin{array}{cc}
x_{3} & x_{0}\\
m_{3} & m_{0}
\end{array}\right]\ddot{\sigma}\left(\tilde{\alpha}_{m_{0}}^{(L)}(x_{0})\right)\ddot{\sigma}\left(\tilde{\alpha}_{m_{1}}^{(L)}(x_{1})\right)\\
\dot{\sigma}\left(\tilde{\alpha}_{m_{2}}^{(L)}(x_{2})\right)\ddot{\sigma}\left(\tilde{\alpha}_{m_{3}}^{(L)}(x_{3})\right)W_{m_{0}k_{0}}^{(L)}W_{m_{1}k_{1}}^{(L)}W_{m_{2}k_{2}}^{(L)}W_{m_{3}k_{3}}^{(L)}
\end{aligned}
\\
 & \begin{aligned}+n_{L}^{-2}\sum_{m_{0},m_{1},m_{2},m_{3}}\left[\begin{array}{cc}
x_{0} & x_{1}\\
m_{0} & m_{1}
\end{array}\right]\left[\begin{array}{cc}
x_{1} & x_{2}\\
m_{1} & m_{2}
\end{array}\right]\left[\begin{array}{ccc}
x_{2} & x_{3} & x_{0}\\
m_{2} & m_{3} & m_{0}
\end{array}\right]\ddot{\sigma}\left(\tilde{\alpha}_{m_{0}}^{(L)}(x_{0})\right)\ddot{\sigma}\left(\tilde{\alpha}_{m_{1}}^{(L)}(x_{1})\right)\\
\ddot{\sigma}\left(\tilde{\alpha}_{m_{2}}^{(L)}(x_{2})\right)\dot{\sigma}\left(\tilde{\alpha}_{m_{3}}^{(L)}(x_{3})\right)W_{m_{0}k_{0}}^{(L)}W_{m_{1}k_{1}}^{(L)}W_{m_{2}k_{2}}^{(L)}W_{m_{3}k_{3}}^{(L)}
\end{aligned}
\\
 & \begin{aligned}+n_{L}^{-2}\sum_{m_{0},m_{1},m_{2},m_{3}}\left[\begin{array}{cc}
x_{1} & x_{2}\\
m_{1} & m_{2}
\end{array}\right]\left[\begin{array}{cc}
x_{2} & x_{3}\\
m_{2} & m_{3}
\end{array}\right]\left[\begin{array}{ccc}
x_{3} & x_{0} & x_{1}\\
m_{3} & m_{0} & m_{1}
\end{array}\right]\dot{\sigma}\left(\tilde{\alpha}_{m_{0}}^{(L)}(x_{0})\right)\ddot{\sigma}\left(\tilde{\alpha}_{m_{1}}^{(L)}(x_{1})\right)\\
\ddot{\sigma}\left(\tilde{\alpha}_{m_{2}}^{(L)}(x_{2})\right)\ddot{\sigma}\left(\tilde{\alpha}_{m_{3}}^{(L)}(x_{3})\right)W_{m_{0}k_{0}}^{(L)}W_{m_{1}k_{1}}^{(L)}W_{m_{2}k_{2}}^{(L)}W_{m_{3}k_{3}}^{(L)}
\end{aligned}
\\
 & \begin{aligned}+n_{L}^{-2}\sum_{m_{0},m_{1},m_{2},m_{3}}\left[\begin{array}{ccc}
x_{0} & x_{1} & x_{2}\\
m_{0} & m_{1} & m_{2}
\end{array}\right]\left[\begin{array}{ccc}
x_{2} & x_{3} & x_{0}\\
m_{2} & m_{3} & m_{0}
\end{array}\right]\ddot{\sigma}\left(\tilde{\alpha}_{m_{0}}^{(L)}(x_{0})\right)\dot{\sigma}\left(\tilde{\alpha}_{m_{1}}^{(L)}(x_{1})\right)\\
\ddot{\sigma}\left(\tilde{\alpha}_{m_{2}}^{(L)}(x_{2})\right)\dot{\sigma}\left(\tilde{\alpha}_{m_{3}}^{(L)}(x_{3})\right)W_{m_{0}k_{0}}^{(L)}W_{m_{1}k_{1}}^{(L)}W_{m_{2}k_{2}}^{(L)}W_{m_{3}k_{3}}^{(L)}
\end{aligned}
\\
 & \begin{aligned}+n_{L}^{-2}\sum_{m_{0},m_{1},m_{2},m_{3}}\left[\begin{array}{ccc}
x_{1} & x_{2} & x_{3}\\
m_{1} & m_{2} & m_{3}
\end{array}\right]\left[\begin{array}{ccc}
x_{3} & x_{0} & x_{1}\\
m_{3} & m_{0} & m_{1}
\end{array}\right]\dot{\sigma}\left(\tilde{\alpha}_{m_{0}}^{(L)}(x_{0})\right)\ddot{\sigma}\left(\tilde{\alpha}_{m_{1}}^{(L)}(x_{1})\right)\\
\dot{\sigma}\left(\tilde{\alpha}_{m_{2}}^{(L)}(x_{2})\right)\ddot{\sigma}\left(\tilde{\alpha}_{m_{3}}^{(L)}(x_{3})\right)W_{m_{0}k_{0}}^{(L)}W_{m_{1}k_{1}}^{(L)}W_{m_{2}k_{2}}^{(L)}W_{m_{3}k_{3}}^{(L)}
\end{aligned}
\\
 & \begin{aligned}+n_{L}^{-2}\sum_{m_{0},m_{1},m_{2},m_{3}}\left[\begin{array}{cccc}
x_{0} & x_{1} & x_{2} & x_{3}\\
m_{0} & m_{1} & m_{2} & m_{3}
\end{array}\right]\left[\begin{array}{cc}
x_{3} & x_{0}\\
m_{3} & m_{0}
\end{array}\right]\ddot{\sigma}\left(\tilde{\alpha}_{m_{0}}^{(L)}(x_{0})\right)\dot{\sigma}\left(\tilde{\alpha}_{m_{1}}^{(L)}(x_{1})\right)\\
\dot{\sigma}\left(\tilde{\alpha}_{m_{2}}^{(L)}(x_{2})\right)\ddot{\sigma}\left(\tilde{\alpha}_{m_{3}}^{(L)}(x_{3})\right)W_{m_{0}k_{0}}^{(L)}W_{m_{1}k_{1}}^{(L)}W_{m_{2}k_{2}}^{(L)}W_{m_{3}k_{3}}^{(L)}
\end{aligned}
\\
 & \begin{aligned}+n_{L}^{-2}\sum_{m_{0},m_{1},m_{2},m_{3}}\left[\begin{array}{cc}
x_{0} & x_{1}\\
m_{0} & m_{1}
\end{array}\right]\left[\begin{array}{cccc}
x_{1} & x_{2} & x_{3} & x_{0}\\
m_{1} & m_{2} & m_{3} & m_{0}
\end{array}\right]\ddot{\sigma}\left(\tilde{\alpha}_{m_{0}}^{(L)}(x_{0})\right)\ddot{\sigma}\left(\tilde{\alpha}_{m_{1}}^{(L)}(x_{1})\right)\\
\dot{\sigma}\left(\tilde{\alpha}_{m_{2}}^{(L)}(x_{2})\right)\dot{\sigma}\left(\tilde{\alpha}_{m_{3}}^{(L)}(x_{3})\right)W_{m_{0}k_{0}}^{(L)}W_{m_{1}k_{1}}^{(L)}W_{m_{2}k_{2}}^{(L)}W_{m_{3}k_{3}}^{(L)}
\end{aligned}
\\
 & \begin{aligned}+n_{L}^{-2}\sum_{m_{0},m_{1},m_{2},m_{3}}\left[\begin{array}{cc}
x_{1} & x_{2}\\
m_{1} & m_{2}
\end{array}\right]\left[\begin{array}{cccc}
x_{2} & x_{3} & x_{0} & x_{1}\\
m_{2} & m_{3} & m_{0} & m_{1}
\end{array}\right]\dot{\sigma}\left(\tilde{\alpha}_{m_{0}}^{(L)}(x_{0})\right)\ddot{\sigma}\left(\tilde{\alpha}_{m_{1}}^{(L)}(x_{1})\right)\\
\ddot{\sigma}\left(\tilde{\alpha}_{m_{2}}^{(L)}(x_{2})\right)\dot{\sigma}\left(\tilde{\alpha}_{m_{3}}^{(L)}(x_{3})\right)W_{m_{0}k_{0}}^{(L)}W_{m_{1}k_{1}}^{(L)}W_{m_{2}k_{2}}^{(L)}W_{m_{3}k_{3}}^{(L)}
\end{aligned}
\\
 & \begin{aligned}+n_{L}^{-2}\sum_{m_{0},m_{1},m_{2},m_{3}}\left[\begin{array}{cc}
x_{2} & x_{3}\\
m_{2} & m_{3}
\end{array}\right]\left[\begin{array}{cccc}
x_{3} & x_{0} & x_{1} & x_{2}\\
m_{3} & m_{0} & m_{1} & m_{2}
\end{array}\right]\dot{\sigma}\left(\tilde{\alpha}_{m_{0}}^{(L)}(x_{0})\right)\dot{\sigma}\left(\tilde{\alpha}_{m_{1}}^{(L)}(x_{1})\right)\\
\ddot{\sigma}\left(\tilde{\alpha}_{m_{2}}^{(L)}(x_{2})\right)\ddot{\sigma}\left(\tilde{\alpha}_{m_{3}}^{(L)}(x_{3})\right)W_{m_{0}k_{0}}^{(L)}W_{m_{1}k_{1}}^{(L)}W_{m_{2}k_{2}}^{(L)}W_{m_{3}k_{3}}^{(L)}
\end{aligned}
\\
 & \begin{aligned}+n_{L}^{-2}\sum_{m,m_{1},m_{2}}\left[\begin{array}{cc}
x_{0} & x_{1}\\
m & m_{1}
\end{array}\right]\left[\begin{array}{cc}
x_{1} & x_{2}\\
m_{1} & m_{2}
\end{array}\right]\left[\begin{array}{cc}
x_{2} & x_{3}\\
m_{2} & m
\end{array}\right]\dot{\sigma}\left(\tilde{\alpha}_{m}^{(L)}(x_{0})\right)\ddot{\sigma}\left(\tilde{\alpha}_{m_{1}}^{(L)}(x_{1})\right)\\
\ddot{\sigma}\left(\tilde{\alpha}_{m_{2}}^{(L)}(x_{2})\right)\dot{\sigma}\left(\tilde{\alpha}_{m}^{(L)}(x_{3})\right)W_{m_{1}k_{1}}^{(L)}W_{m_{2}k_{2}}^{(L)}\delta_{k_{0}k_{3}}
\end{aligned}
\\
 & \begin{aligned}+n_{L}^{-2}\sum_{m,m_{2},m_{3}}\left[\begin{array}{cc}
x_{1} & x_{2}\\
m & m_{2}
\end{array}\right]\left[\begin{array}{cc}
x_{2} & x_{3}\\
m_{2} & m_{3}
\end{array}\right]\left[\begin{array}{cc}
x_{3} & x_{0}\\
m_{3} & m
\end{array}\right]\dot{\sigma}\left(\tilde{\alpha}_{m}^{(L)}(x_{0})\right)\dot{\sigma}\left(\tilde{\alpha}_{m}^{(L)}(x_{1})\right)\\
\ddot{\sigma}\left(\tilde{\alpha}_{m_{2}}^{(L)}(x_{2})\right)\ddot{\sigma}\left(\tilde{\alpha}_{m_{3}}^{(L)}(x_{3})\right)W_{m_{2}k_{2}}^{(L)}W_{m_{3}k_{3}}^{(L)}\delta_{k_{0}k_{1}}
\end{aligned}
\\
 & \begin{aligned}+n_{L}^{-2}\sum_{m,m_{3},m_{0}}\left[\begin{array}{cc}
x_{0} & x_{1}\\
m_{0} & m
\end{array}\right]\left[\begin{array}{cc}
x_{2} & x_{3}\\
m & m_{3}
\end{array}\right]\left[\begin{array}{cc}
x_{3} & x_{0}\\
m_{3} & m_{0}
\end{array}\right]\ddot{\sigma}\left(\tilde{\alpha}_{m_{0}}^{(L)}(x_{0})\right)\dot{\sigma}\left(\tilde{\alpha}_{m}^{(L)}(x_{1})\right)\\
\dot{\sigma}\left(\tilde{\alpha}_{m}^{(L)}(x_{2})\right)\ddot{\sigma}\left(\tilde{\alpha}_{m_{3}}^{(L)}(x_{3})\right)W_{m_{0}k_{0}}^{(L)}W_{m_{3}k_{3}}^{(L)}\delta_{k_{1}k_{2}}
\end{aligned}
\\
 & \begin{aligned}+n_{L}^{-2}\sum_{m,m_{0},m_{1}}\left[\begin{array}{cc}
x_{0} & x_{1}\\
m_{0} & m_{1}
\end{array}\right]\left[\begin{array}{cc}
x_{1} & x_{2}\\
m_{1} & m
\end{array}\right]\left[\begin{array}{cc}
x_{3} & x_{0}\\
m & m_{0}
\end{array}\right]\ddot{\sigma}\left(\tilde{\alpha}_{m_{0}}^{(L)}(x_{0})\right)\ddot{\sigma}\left(\tilde{\alpha}_{m_{1}}^{(L)}(x_{1})\right)\\
\dot{\sigma}\left(\tilde{\alpha}_{m}^{(L)}(x_{2})\right)\dot{\sigma}\left(\tilde{\alpha}_{m}^{(L)}(x_{3})\right)W_{m_{0}k_{0}}^{(L)}W_{m_{1}k_{1}}^{(L)}\delta_{k_{2}k_{3}}
\end{aligned}
\\
 & \begin{aligned}+n_{L}^{-2}\sum_{m,m_{1},m_{2}}\left[\begin{array}{ccc}
x_{0} & x_{1} & x_{2}\\
m & m_{1} & m_{2}
\end{array}\right]\left[\begin{array}{cc}
x_{2} & x_{3}\\
m_{2} & m
\end{array}\right]\dot{\sigma}\left(\tilde{\alpha}_{m}^{(L)}(x_{0})\right)\dot{\sigma}\left(\tilde{\alpha}_{m_{1}}^{(L)}(x_{1})\right)\\
\ddot{\sigma}\left(\tilde{\alpha}_{m_{2}}^{(L)}(x_{2})\right)\dot{\sigma}\left(\tilde{\alpha}_{m}^{(L)}(x_{3})\right)W_{m_{1}k_{1}}^{(L)}W_{m_{2}k_{2}}^{(L)}\delta_{k_{0}k_{3}}
\end{aligned}
\\
 & \begin{aligned}+n_{L}^{-2}\sum_{m,m_{2},m_{3}}\left[\begin{array}{ccc}
x_{1} & x_{2} & x_{3}\\
m & m_{2} & m_{3}
\end{array}\right]\left[\begin{array}{cc}
x_{3} & x_{0}\\
m_{3} & m
\end{array}\right]\dot{\sigma}\left(\tilde{\alpha}_{m}^{(L)}(x_{0})\right)\dot{\sigma}\left(\tilde{\alpha}_{m}^{(L)}(x_{1})\right)\dot{\sigma}\left(\tilde{\alpha}_{m_{2}}^{(L)}(x_{2})\right)\\
\ddot{\sigma}\left(\tilde{\alpha}_{m_{3}}^{(L)}(x_{3})\right)W_{m_{2}k_{2}}^{(L)}W_{m_{3}k_{3}}^{(L)}\delta_{k_{0}k_{1}}
\end{aligned}
\\
 & \begin{aligned}+n_{L}^{-2}\sum_{m,m_{3},m_{0}}\left[\begin{array}{cc}
x_{0} & x_{1}\\
m_{0} & m
\end{array}\right]\left[\begin{array}{ccc}
x_{2} & x_{3} & x_{0}\\
m & m_{3} & m_{0}
\end{array}\right]\ddot{\sigma}\left(\tilde{\alpha}_{m_{0}}^{(L)}(x_{0})\right)\dot{\sigma}\left(\tilde{\alpha}_{m}^{(L)}(x_{1})\right)\\
\dot{\sigma}\left(\tilde{\alpha}_{m}^{(L)}(x_{2})\right)\dot{\sigma}\left(\tilde{\alpha}_{m_{3}}^{(L)}(x_{3})\right)W_{m_{0}k_{0}}^{(L)}W_{m_{3}k_{3}}^{(L)}\delta_{k_{1}k_{2}}
\end{aligned}
\\
 & \begin{aligned}+n_{L}^{-2}\sum_{m,m_{0},m_{1}}\left[\begin{array}{cc}
x_{1} & x_{2}\\
m_{1} & m
\end{array}\right]\left[\begin{array}{ccc}
x_{3} & x_{0} & x_{1}\\
m & m_{0} & m_{1}
\end{array}\right]\dot{\sigma}\left(\tilde{\alpha}_{m_{0}}^{(L)}(x_{0})\right)\ddot{\sigma}\left(\tilde{\alpha}_{m_{1}}^{(L)}(x_{1})\right)\\
\dot{\sigma}\left(\tilde{\alpha}_{m}^{(L)}(x_{2})\right)\dot{\sigma}\left(\tilde{\alpha}_{m}^{(L)}(x_{3})\right)W_{m_{0}k_{0}}^{(L)}W_{m_{1}k_{1}}^{(L)}\delta_{k_{2}k_{3}}
\end{aligned}
\\
 & \begin{aligned}+n_{L}^{-2}\sum_{m,m_{1},m_{2}}\left[\begin{array}{cc}
x_{0} & x_{1}\\
m & m_{1}
\end{array}\right]\left[\begin{array}{ccc}
x_{1} & x_{2} & x_{3}\\
m_{1} & m_{2} & m
\end{array}\right]\dot{\sigma}\left(\tilde{\alpha}_{m}^{(L)}(x_{0})\right)\ddot{\sigma}\left(\tilde{\alpha}_{m_{1}}^{(L)}(x_{1})\right)\\
\dot{\sigma}\left(\tilde{\alpha}_{m_{2}}^{(L)}(x_{2})\right)\dot{\sigma}\left(\tilde{\alpha}_{m}^{(L)}(x_{3})\right)W_{m_{1}k_{1}}^{(L)}W_{m_{2}k_{2}}^{(L)}\delta_{k_{0}k_{3}}
\end{aligned}
\\
 & \begin{aligned}+n_{L}^{-2}\sum_{m,m_{2},m_{3}}\left[\begin{array}{cc}
x_{1} & x_{2}\\
m & m_{2}
\end{array}\right]\left[\begin{array}{ccc}
x_{2} & x_{3} & x_{0}\\
m_{2} & m_{3} & m
\end{array}\right]\dot{\sigma}\left(\tilde{\alpha}_{m}^{(L)}(x_{0})\right)\dot{\sigma}\left(\tilde{\alpha}_{m}^{(L)}(x_{1})\right)\\
\ddot{\sigma}\left(\tilde{\alpha}_{m_{2}}^{(L)}(x_{2})\right)\dot{\sigma}\left(\tilde{\alpha}_{m_{3}}^{(L)}(x_{3})\right)W_{m_{2}k_{2}}^{(L)}W_{m_{3}k_{3}}^{(L)}\delta_{k_{0}k_{1}}
\end{aligned}
\\
 & \begin{aligned}+n_{L}^{-2}\sum_{m,m_{3},m_{0}}\left[\begin{array}{cc}
x_{2} & x_{3}\\
m & m_{3}
\end{array}\right]\left[\begin{array}{ccc}
x_{3} & x_{0} & x_{1}\\
m_{3} & m_{0} & m
\end{array}\right]\dot{\sigma}\left(\tilde{\alpha}_{m_{0}}^{(L)}(x_{0})\right)\dot{\sigma}\left(\tilde{\alpha}_{m}^{(L)}(x_{1})\right)\\
\dot{\sigma}\left(\tilde{\alpha}_{m}^{(L)}(x_{2})\right)\ddot{\sigma}\left(\tilde{\alpha}_{m_{3}}^{(L)}(x_{3})\right)W_{m_{0}k_{0}}^{(L)}W_{m_{3}k_{3}}^{(L)}\delta_{k_{1}k_{2}}
\end{aligned}
\\
 & \begin{aligned}+n_{L}^{-2}\sum_{m,m_{0},m_{1}}\left[\begin{array}{ccc}
x_{0} & x_{1} & x_{2}\\
m_{0} & m_{1} & m
\end{array}\right]\left[\begin{array}{cc}
x_{3} & x_{0}\\
m & m_{0}
\end{array}\right]\ddot{\sigma}\left(\tilde{\alpha}_{m_{0}}^{(L)}(x_{0})\right)\dot{\sigma}\left(\tilde{\alpha}_{m_{1}}^{(L)}(x_{1})\right)\\
\dot{\sigma}\left(\tilde{\alpha}_{m}^{(L)}(x_{2})\right)\dot{\sigma}\left(\tilde{\alpha}_{m}^{(L)}(x_{3})\right)W_{m_{0}k_{0}}^{(L)}W_{m_{1}k_{1}}^{(L)}\delta_{k_{2}k_{3}}
\end{aligned}
\\
 & \begin{aligned}+n_{L}^{-2}\sum_{m,m_{1},m_{2}}\left[\begin{array}{cccc}
x_{0} & x_{1} & x_{2} & x_{3}\\
m & m_{1} & m_{2} & m
\end{array}\right]\dot{\sigma}\left(\tilde{\alpha}_{m}^{(L)}(x_{0})\right)\dot{\sigma}\left(\tilde{\alpha}_{m_{1}}^{(L)}(x_{1})\right)\dot{\sigma}\left(\tilde{\alpha}_{m_{2}}^{(L)}(x_{2})\right)\dot{\sigma}\left(\tilde{\alpha}_{m}^{(L)}(x_{3})\right)\\
W_{m_{1}k_{1}}^{(L)}W_{m_{2}k_{2}}^{(L)}\delta_{k_{0}k_{3}}
\end{aligned}
\\
 & \begin{aligned}+n_{L}^{-2}\sum_{m,m_{2},m_{3}}\left[\begin{array}{cccc}
x_{1} & x_{2} & x_{3} & x_{0}\\
m & m_{2} & m_{3} & m
\end{array}\right]\dot{\sigma}\left(\tilde{\alpha}_{m}^{(L)}(x_{0})\right)\dot{\sigma}\left(\tilde{\alpha}_{m}^{(L)}(x_{1})\right)\dot{\sigma}\left(\tilde{\alpha}_{m_{2}}^{(L)}(x_{2})\right)\dot{\sigma}\left(\tilde{\alpha}_{m_{3}}^{(L)}(x_{3})\right)\\
W_{m_{2}k_{2}}^{(L)}W_{m_{3}k_{3}}^{(L)}\delta_{k_{0}k_{1}}
\end{aligned}
\\
 & \begin{aligned}+n_{L}^{-2}\sum_{m,m_{3},m_{0}}\left[\begin{array}{cccc}
x_{2} & x_{3} & x_{0} & x_{1}\\
m & m_{3} & m_{0} & m
\end{array}\right]\dot{\sigma}\left(\tilde{\alpha}_{m_{0}}^{(L)}(x_{0})\right)\dot{\sigma}\left(\tilde{\alpha}_{m}^{(L)}(x_{1})\right)\dot{\sigma}\left(\tilde{\alpha}_{m}^{(L)}(x_{2})\right)\dot{\sigma}\left(\tilde{\alpha}_{m_{3}}^{(L)}(x_{3})\right)\\
W_{m_{0}k_{0}}^{(L)}W_{m_{3}k_{3}}^{(L)}\delta_{k_{1}k_{2}}
\end{aligned}
\\
 & \begin{aligned}+n_{L}^{-2}\sum_{m,m_{0},m_{1}}\left[\begin{array}{cccc}
x_{3} & x_{0} & x_{1} & x_{2}\\
m & m_{0} & m_{1} & m
\end{array}\right]\dot{\sigma}\left(\tilde{\alpha}_{m_{0}}^{(L)}(x_{0})\right)\dot{\sigma}\left(\tilde{\alpha}_{m_{1}}^{(L)}(x_{1})\right)\dot{\sigma}\left(\tilde{\alpha}_{m}^{(L)}(x_{2})\right)\dot{\sigma}\left(\tilde{\alpha}_{m}^{(L)}(x_{3})\right)\\
W_{m_{0}k_{0}}^{(L)}W_{m_{1}k_{1}}^{(L)}\delta_{k_{2}k_{3}}
\end{aligned}
\\
 & \begin{aligned}+n_{L}^{-2}\sum_{m,m'}\left[\begin{array}{cc}
x_{0} & x_{1}\\
m & m'
\end{array}\right]\left[\begin{array}{cc}
x_{2} & x_{3}\\
m' & m
\end{array}\right]\dot{\sigma}\left(\tilde{\alpha}_{m}^{(L)}(x_{0})\right)\dot{\sigma}\left(\tilde{\alpha}_{m'}^{(L)}(x_{1})\right)\dot{\sigma}\left(\tilde{\alpha}_{m'}^{(L)}(x_{2})\right)\dot{\sigma}\left(\tilde{\alpha}_{m}^{(L)}(x_{3})\right)\\
\delta_{k_{0}k_{1}}\delta_{k_{2}k_{3}}
\end{aligned}
\\
 & \begin{aligned}+n_{L}^{-2}\sum_{m,m'}\left[\begin{array}{cc}
x_{1} & x_{2}\\
m & m'
\end{array}\right]\left[\begin{array}{cc}
x_{3} & x_{0}\\
m' & m
\end{array}\right]\dot{\sigma}\left(\tilde{\alpha}_{m}^{(L)}(x_{0})\right)\dot{\sigma}\left(\tilde{\alpha}_{m}^{(L)}(x_{1})\right)\dot{\sigma}\left(\tilde{\alpha}_{m'}^{(L)}(x_{2})\right)\dot{\sigma}\left(\tilde{\alpha}_{m'}^{(L)}(x_{3})\right)\\
\delta_{k_{0}k_{3}}\delta_{k_{1}k_{2}}
\end{aligned}
\end{align*}

Even though this is a very large formula one can notice that most
terms are ``rotation of each other''. Moreover, as $n_{1},...,n_{L-1}\to\infty$,
all terms containing either an $\Psi^{(L)}$, an $\Omega^{(L)}$ or
a $\Gamma^{(L)}$ vanish. For the remaining terms, we may replace
the NTKs $\Theta^{(L)}$ by their limit and as a result $\Psi_{k_{0},k_{1},k_{2},k_{3}}^{(L+1)}(x_{i_{0}},x_{i_{1}},x_{i_{2}},x_{i_{3}})$
converges to
\begin{align*}
 & \begin{aligned}n_{L}^{-2}\sum_{m}\Theta_{\infty}^{(L)}(x_{0},x_{1})\Theta_{\infty}^{(L)}(x_{1},x_{2})\Theta_{\infty}^{(L)}(x_{2},x_{3})\Theta_{\infty}^{(L)}(x_{3},x_{0})\ddot{\sigma}\left(\tilde{\alpha}_{m}^{(L)}(x_{0})\right)\ddot{\sigma}\left(\tilde{\alpha}_{m}^{(L)}(x_{1})\right)\\
\ddot{\sigma}\left(\tilde{\alpha}_{m}^{(L)}(x_{2})\right)\ddot{\sigma}\left(\tilde{\alpha}_{m}^{(L)}(x_{3})\right)W_{mk_{0}}^{(L)}W_{mk_{1}}^{(L)}W_{mk_{2}}^{(L)}W_{mk_{3}}^{(L)}
\end{aligned}
\\
 & \begin{aligned}+n_{L}^{-2}\sum_{m}\Theta_{\infty}^{(L)}(x_{0},x_{1})\Theta_{\infty}^{(L)}(x_{1},x_{2})\Theta_{\infty}^{(L)}(x_{2},x_{3})\dot{\sigma}\left(\tilde{\alpha}_{m}^{(L)}(x_{0})\right)\ddot{\sigma}\left(\tilde{\alpha}_{m}^{(L)}(x_{1})\right)\\
\ddot{\sigma}\left(\tilde{\alpha}_{m}^{(L)}(x_{2})\right)\dot{\sigma}\left(\tilde{\alpha}_{m}^{(L)}(x_{3})\right)W_{mk_{1}}^{(L)}W_{mk_{2}}^{(L)}\delta_{k_{0}k_{3}}
\end{aligned}
\\
 & \begin{aligned}+n_{L}^{-2}\sum_{m}\Theta_{\infty}^{(L)}(x_{1},x_{2})\Theta_{\infty}^{(L)}(x_{2},x_{3})\Theta_{\infty}^{(L)}(x_{3},x_{0})\dot{\sigma}\left(\tilde{\alpha}_{m}^{(L)}(x_{0})\right)\dot{\sigma}\left(\tilde{\alpha}_{m}^{(L)}(x_{1})\right)\\
\ddot{\sigma}\left(\tilde{\alpha}_{m}^{(L)}(x_{2})\right)\ddot{\sigma}\left(\tilde{\alpha}_{m}^{(L)}(x_{3})\right)W_{mk_{2}}^{(L)}W_{mk_{3}}^{(L)}\delta_{k_{0}k_{1}}
\end{aligned}
\\
 & \begin{aligned}+n_{L}^{-2}\sum_{m}\Theta_{\infty}^{(L)}(x_{0},x_{1})\Theta_{\infty}^{(L)}(x_{2},x_{3})\Theta_{\infty}^{(L)}(x_{3},x_{0})\ddot{\sigma}\left(\tilde{\alpha}_{m}^{(L)}(x_{0})\right)\dot{\sigma}\left(\tilde{\alpha}_{m}^{(L)}(x_{1})\right)\\
\dot{\sigma}\left(\tilde{\alpha}_{m}^{(L)}(x_{2})\right)\ddot{\sigma}\left(\tilde{\alpha}_{m}^{(L)}(x_{3})\right)W_{mk_{0}}^{(L)}W_{mk_{3}}^{(L)}\delta_{k_{1}k_{2}}
\end{aligned}
\\
 & \begin{aligned}+n_{L}^{-2}\sum_{m}\Theta_{\infty}^{(L)}(x_{0},x_{1})\Theta_{\infty}^{(L)}(x_{1},x_{2})\Theta_{\infty}^{(L)}(x_{3},x_{0})\ddot{\sigma}\left(\tilde{\alpha}_{m}^{(L)}(x_{0})\right)\ddot{\sigma}\left(\tilde{\alpha}_{m}^{(L)}(x_{1})\right)\\
\dot{\sigma}\left(\tilde{\alpha}_{m}^{(L)}(x_{2})\right)\dot{\sigma}\left(\tilde{\alpha}_{m}^{(L)}(x_{3})\right)W_{mk_{0}}^{(L)}W_{mk_{1}}^{(L)}\delta_{k_{2}k_{3}}
\end{aligned}
\\
 & \begin{aligned}+n_{L}^{-2}\sum_{m}\Theta_{\infty}^{(L)}(x_{0},x_{1})\Theta_{\infty}^{(L)}(x_{2},x_{3})\dot{\sigma}\left(\tilde{\alpha}_{m}^{(L)}(x_{0})\right)\dot{\sigma}\left(\tilde{\alpha}_{m}^{(L)}(x_{1})\right)\\
\dot{\sigma}\left(\tilde{\alpha}_{m}^{(L)}(x_{2})\right)\dot{\sigma}\left(\tilde{\alpha}_{m}^{(L)}(x_{3})\right)\delta_{k_{0}k_{1}}\delta_{k_{2}k_{3}}
\end{aligned}
\\
 & \begin{aligned}+n_{L}^{-2}\sum_{m}\Theta_{\infty}^{(L)}(x_{1},x_{2})\Theta_{\infty}^{(L)}(x_{3},x_{0})\dot{\sigma}\left(\tilde{\alpha}_{m}^{(L)}(x_{0})\right)\dot{\sigma}\left(\tilde{\alpha}_{m}^{(L)}(x_{1})\right)\\
\dot{\sigma}\left(\tilde{\alpha}_{m}^{(L)}(x_{2})\right)\dot{\sigma}\left(\tilde{\alpha}_{m}^{(L)}(x_{3})\right)\delta_{k_{0}k_{3}}\delta_{k_{1}k_{2}}
\end{aligned}
\end{align*}
And all these sums vanish as $n_{L}\to\infty$ thanks to the prefactor
$n_{L}^{-2}$, proving the vanishing of $\Psi_{k_{0},k_{1},k_{2},k_{3}}^{(L+1)}(x_{i_{0}},x_{i_{1}},x_{i_{2}},x_{i_{3}})$
in the infinite width limit.

During training, the activations $\tilde{\alpha}_{m}^{(L)}(x)$ and
weights $W_{mk}^{(L)}$ move at a rate of $\nicefrac{1}{\sqrt{n_{L}}}$
which induces a change to $\Psi^{(L+1)}$ of order $n_{L}^{-\nicefrac{3}{2}}$
which vanishes in the infinite width limit.
\end{proof}

\section{Orthogonality of $I$ and $S$\label{sec:Orthogonality-of-I-S}}

From Lemma \ref{lem:Gamma_vanish} and the vanishing of the tensor
$\Gamma^{(L)}$ as proven in Lemma \ref{lem:Gamma_vanish}, we can
easily prove the orthogonality of $I$ and $S$ of Proposition \ref{prop:orthogonality_I_S}:
\begin{prop}
\label{prop:orthogonality_I_S}For any loss $C$ with BGOSS and $\sigma\in C_{b}^{4}(\mathbb{R})$,
we have uniformly over $[0,T]$
\[
\lim_{n_{L-1}\to\infty}\cdots\lim_{n_{1}\to\infty}\|IS\|_{F}=0.
\]
As a consequence $\lim_{n_{L-1}\to\infty}\cdots\lim_{n_{1}\to\infty}\mathrm{Tr}\left(\left[I+S\right]^{k}\right)-\left[\mathrm{Tr}\left(I^{k}\right)+\mathrm{Tr}\left(S^{k}\right)\right]=0$.
\end{prop}
\begin{proof}
The Frobenius norm of $IS$ is equal to
\begin{align*}
\left\Vert IS\right\Vert _{F}^{2} & =\left\Vert \mathcal{D}Y\mathcal{H}C\left(\mathcal{D}Y\right)^{T}\left(\nabla C\cdot\mathcal{H}Y\right)\right\Vert _{F}^{2}\\
 & =\sum_{p_{1},p_{2}=1}^{P}\left(\sum_{p=1}^{P}\sum_{i_{1},i_{2}=1}^{N}\sum_{k_{1},k_{2}=1}^{n_{L}}\partial_{\theta_{p_{1}}}f_{\theta,k_{1}}(x_{i_{1}})c''_{k_{1}}(x_{i_{1}})\partial_{\theta_{p}}f_{\theta,k_{1}}(x_{i_{1}})\partial_{\theta_{p},\theta_{p_{3}}}^{2}f_{\theta,k_{2}}(x_{2})(x_{i_{2}})c'_{k_{2}}(x_{i_{2}})\right)^{2}\\
 & =\sum_{i_{1},i_{2},i'_{1},i'_{2}=1}^{N}\sum_{k_{1},k_{2},k'_{1},k'_{2}=1}^{n_{L}}c''_{k_{1}}(x_{i_{1}})c''_{k'_{1}}(x_{i'_{1}})c'_{k_{2}}(x_{i_{2}})c'_{k'_{2}}(x_{i'_{2}})\Theta_{k_{1},k'_{1}}(x_{i_{1}},x_{i'_{1}})\Gamma_{k_{1},k_{2},k'_{2},k'_{1}}(x_{i_{1}},x_{i_{2}},x_{i'_{2}},x_{i'_{1}})
\end{align*}
and $\Gamma$ vanishes as $n_{1},...,n_{L-1}\to\infty$ by Lemma \ref{lem:Gamma_vanish}.

The $k$-th moment of the sum $\mathrm{Tr}\left(I+S\right)^{k}$ is
equal to the sum over all $\mathrm{Tr}\left(A_{1}\cdots A_{k}\right)$
for any word $A_{1}\ldots A_{k}$ of $A_{i}\in\left\{ I,S\right\} $.
The difference $\mathrm{Tr}\left(\left[I+S\right]^{k}\right)-\left[\mathrm{Tr}\left(I^{k}\right)+\mathrm{Tr}\left(S^{k}\right)\right]$
is hence equal to the sum over all mixed words, i.e. words $A_{1}\ldots A_{k}$
which contain at least one $I$ and one $S$. Such words must contain
two consecutive terms $A_{m}A_{m+1}$ one equal to $I$ and the other
equal to $S$. We can then bound the trace by
\[
\left|\mathrm{Tr}\left(A_{1}\cdots A_{k}\right)\right|\leq\left\Vert A_{1}\right\Vert _{F}\cdots\left\Vert A_{m-1}\right\Vert _{F}\left\Vert A_{m}A_{m+1}\right\Vert _{F}\left\Vert A_{m+2}\right\Vert _{F}\cdots\left\Vert A_{k}\right\Vert _{F}
\]
which vanishes in the infinite width limit because $\left\Vert I\right\Vert _{F}$
and $\left\Vert S\right\Vert _{F}$ are bounded and $\left\Vert A_{m}A_{m+1}\right\Vert _{F}=\left\Vert IS\right\Vert _{F}$
vanishes.
\end{proof}

\end{document}